\documentclass[lettersize,journal]{IEEEtran}

\IEEEoverridecommandlockouts                              

\usepackage{graphicx}
\usepackage{cite}
\usepackage{amsmath}
\usepackage{adjustbox}
\usepackage{array}
\usepackage{subfigure}
\usepackage{multicol}
\usepackage{float}
\usepackage{setspace}
\usepackage{bm}
\usepackage{amsfonts}
\usepackage{amsthm}
\usepackage{makecell}
\usepackage{multirow}
\usepackage[hidelinks]{hyperref}
\usepackage{bm}
\usepackage{color}
\allowdisplaybreaks

\newtheorem{lem}{Lemma}
\newtheorem{rem}{Remark}

\newtheorem{prop}{Proposition}
\newtheorem{assump}{Assumption}

\title{\LARGE \bf
Longitudinal-Motion-Aware Lateral Control for Autonomous Vehicles: A Robust Nonlinear Control Framework
}

\author{Sixu Li,~\IEEEmembership{Student Member,~IEEE}, Nitesh Kumar, Reyshwanth Ganeshan,  Sivakumar Rathinam,~\IEEEmembership{Senior Member,~IEEE}, Swaroop Darbha,~\IEEEmembership{Fellow,~IEEE}, Yang Zhou,~\IEEEmembership{Member,~IEEE}
\thanks{(Corresponding authors: Swaroop Darbha, Yang Zhou)}
\thanks{The authors are with Texas A$\&$M University, College Station, TX 77843, USA (e-mail: sixuli@tamu.edu; niteshk@tamu.edu; reyshwanth@tamu.edu; srathinam@tamu.edu; dswaroop@tamu.edu; yangzhou295@tamu.edu).}}%

\begin{document}
\maketitle
\thispagestyle{empty}
\pagestyle{empty}

\begin{abstract}
As autonomous vehicles (AVs) operate in increasingly dynamic traffic conditions, lateral control must be performed while longitudinal speed and acceleration vary. Yet many existing lateral controllers rely on constant-speed or operating-point-based assumptions, which can degrade performance during transient longitudinal maneuvers. Moreover, most methods assume precisely known vehicle parameters, despite real-world parametric uncertainties. To address these limitations, this paper presents a longitudinal-motion-aware robust nonlinear lateral control framework for AVs. It first derives a tracking error model that depends on varying longitudinal speed and acceleration. Using this model, feedback linearization is employed to obtain a linear input–output relation for lateral error tracking while embedding longitudinal motion into the control law. The resulting internal dynamics are then analyzed to ensure overall system stability. To address parameter uncertainty, two robust control designs with distinct implementation trade-offs are proposed: (i) a Lyapunov redesign (LR) approach inspired by sliding mode control, and (ii) an incremental nonlinear dynamic inversion (INDI) method. Both are rigorously analyzed and proven to ensure ultimate boundedness, with key robustness-tuning parameters explicitly identified. Simulations demonstrate enhanced tracking accuracy, consistent performance across varying speeds and accelerations, and robustness to model uncertainties, while also examining the effects of the robustness-related parameters. Real-vehicle tests further confirm real-time implementation and practical path-tracking performance on actual hardware.
\end{abstract}

\begin{IEEEkeywords}
Lateral Control, Nonlinear Dynamics and Robustness, Optimization and control, Autonomous driving.
\end{IEEEkeywords}

\section{Introduction}
\IEEEPARstart{T}{he} rapid advancement of autonomous vehicle (AV) technologies can be traced back to early initiatives such as California PATH \cite{swaroop1994string,peng1992vehicle,rajamani2000demonstration}, with a notable boost from competitions like DARPA \cite{buehler2009darpa}. The field remains active and has expanded considerably, largely due to AVs' potential to enhance traffic safety \cite{rajamani2011vehicle,li2026lateral} and smooth traffic flow \cite{li2025nonlinear,stern2018dissipation}. Among the core modules of an AV system, vehicle control plays a central role by converting planned motion into physically realizable actions while maintaining stability and safety. It is commonly divided into longitudinal control \cite{swaroop1994string,stern2018dissipation,li2024sequencing}, which regulates speed and spacing, and lateral (or path-tracking) control \cite{liu2020lateral,peng1992vehicle,han2024uniform}, which governs steering and path-tracking behavior. This paper focuses on lateral control, addressing the problem of accurately tracking a desired path.

In realistic driving scenarios, lateral control must be performed while longitudinal speed and acceleration vary \cite{li2024sequencing,du2011velocity}, rather than remain constant. At the same time, practical implementation is affected by uncertainty in vehicle parameters such as mass, yaw inertia, and tire cornering stiffness, arising from variable loading and road conditions \cite{zhang2025adaptive}. These two factors together make it important to develop lateral control methods that explicitly account for longitudinal motion variation while remaining robust to parameter mismatches.

\subsection{Related Work and Motivation}
A standard foundation for lateral control design is the bicycle dynamics model \cite{rajamani2011vehicle,liu2020lateral}, known for its simplicity and ability to effectively capture essential lateral and yaw dynamics. This model characterizes vehicle lateral dynamics with two degrees of freedom: lateral motion and yaw rotation. From the bicycle model, tracking error models are developed \cite{peng1992vehicle,tagne2015design}, using the steering angle as the control input while treating the desired yaw rate as an exogenous signal.

To simplify controller design, tracking error models are often reduced to \textit{linear time-invariant} (LTI) form by assuming a \textit{constant} longitudinal speed. While this assumption eases analysis and controller synthesis, it limits realism because vehicle speeds typically vary in real traffic \cite{tian2025physically,li2025nonlinear}. This is evident in Fig.~\ref{fig: ngsim}, which shows NGSIM I-80 trajectories with longitudinal speeds ranging from about 0 to 60 ft/s. To enhance practicality within an LTI setting, methods such as gain scheduling \cite{peng1992vehicle}, frozen-parameter stabilization over multiple constant speeds \cite{liu2020lateral}, and passivity-based design based on the LTI model \cite{tagne2015design} have been proposed. However, these methods do not explicitly address \textit{transient} speed changes, which can degrade lateral tracking and even induce oscillations.
\begin{figure}[htb!]
    \centering
    \includegraphics[width = 0.99\linewidth]{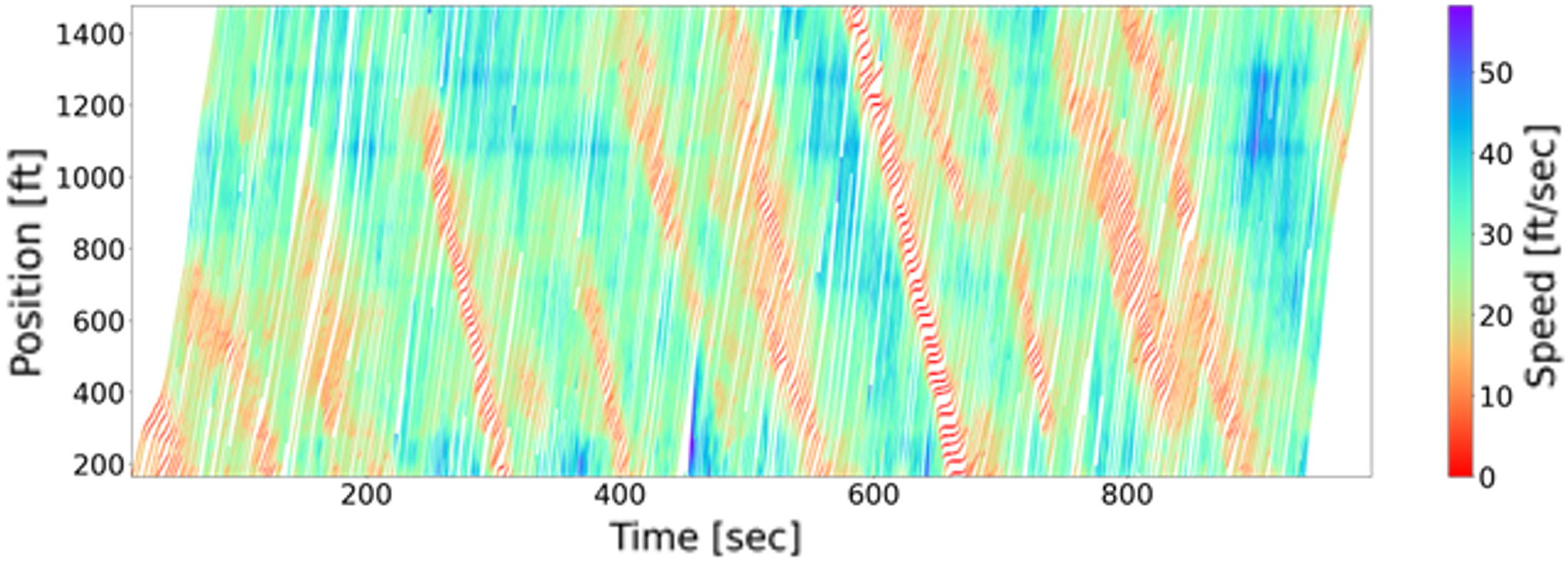}
    \caption{Recorded vehicle trajectories on the I-80 freeway}
    \label{fig: ngsim}
\end{figure}

To address this limitation, numerous studies have modeled the tracking error dynamics as a \textit{linear time-varying} (LTV) system. Representative approaches include Lyapunov-based stabilization for varying longitudinal speed \cite{liu2020robust,rajamani2011vehicle}, robust gain scheduling \cite{zhang2015vehicle,cao2023gain}, and LMI-based controller design \cite{du2011velocity}. Despite these advances, two main drawbacks remain: (i) the LTV model is typically obtained by treating the constant speed in the LTI formulation as a time-varying parameter, without revisiting the modeling simplifications introduced by the constant-speed assumption; and (ii) longitudinal speed is typically treated either as an uncertain parameter within a prescribed range or as a variable defining the controller operating point, rather than being explicitly incorporated into the control law, which may compromise performance.

These limitations mainly arise from enforcing a control-oriented linear model to keep standard linear design tools applicable. Yet, for the tracking error models widely used in the literature \cite{darbha2025robust,jiang2018lateral}, treating longitudinal speed as an external input implies that the error dynamics are inherently \textit{nonlinear}, since the state equations contain terms that depend jointly on the states and the input. This suggests adopting nonlinear control techniques \cite{khalil2002nonlinear}. In this direction, Jiang and Astolfi proposed backstepping- and forwarding-based lateral controllers \cite{jiang2018lateral}, while Tagne et al. explored passivity-based, sliding-mode, and immersion-and-invariance designs \cite{tagne2015design}. Yet both studies assumed constant longitudinal speed and still treated the tracking-error system as LTI, despite the nonlinear coupling. Moreover, \cite{tagne2015design} applied input-output feedback linearization without examining the internal dynamics stability. In \cite{seo2022safety}, a differential-flatness-based method was presented, but it relies on a lateral-speed integration reference that is difficult to obtain in practice, and its stability was not established.

In practice, uncertainties in vehicle parameters such as mass, moment of inertia, and tire cornering stiffness often arise due to variable load and road conditions \cite{darbha2025robust}. Although some nonlinear lateral control studies have addressed robustness against modeling uncertainty and disturbances, such treatment remains limited in the literature. For example, robust nonlinear controllers were employed in \cite{tagne2015design}, but the analysis, as mentioned before, relied on the constant-speed assumption and treated the model as LTI. Internal dynamics were also not analyzed, leaving a comprehensive robustness analysis under varying longitudinal motion unaddressed. Related nonlinear work has also considered disturbance-rejection and constraint-handling through observer- and Barrier Lyapunov Function-based designs \cite{hwang2020robust}, but still under constant-speed formulations rather than a longitudinal-motion-aware tracking framework.


\subsection{Proposed Framework and Contributions}
To address the gaps, this paper presents a robust nonlinear lateral control framework with longitudinal motion awareness. We first derive a control-oriented tracking error model that explicitly accounts for varying longitudinal speed and acceleration. Based on this model, feedback linearization is applied to obtain a control law that directly incorporates longitudinal motion information, and the resulting internal dynamics are explicitly analyzed to ensure overall stability. 

To address parametric uncertainty, we propose two alternative robustification methods. The first is a Lyapunov redesign (LR) scheme that adds a corrective term, derived from the Lyapunov function, to the nominal control law; to reduce chattering \cite{khalil2002nonlinear} without inducing excessively aggressive control, we propose a sliding-mode-inspired LR structure. The second is an incremental nonlinear dynamic inversion (INDI) method, which reduces reliance on an exact model by substituting model-based terms with measurements of the second derivative of the lateral error. These methods entail different implementation trade-offs: LR depends more on parameter information, whereas INDI depends more on sensing.

Both methods are rigorously analyzed and shown to ensure ultimate boundedness, with the key robustness-shaping parameters identified. Simulation results demonstrate improved tracking accuracy and robustness, and further illustrate the effects of the robustness-shaping parameters. The proposed controllers are further validated on a real vehicle.

The main contributions of this paper are as follows:

1) A control-oriented tracking error model that explicitly accounts for varying longitudinal speed and acceleration.

2) A longitudinal-motion-aware lateral control design via feedback linearization, together with an explicit analysis of the resulting internal dynamics.

3) Two robustification methods, LR and INDI, for lateral control under parametric uncertainty, with different trade-offs in model dependence and sensing requirements.

4) Rigorous ultimate boundedness analysis for both LR and INDI. The LR design presents a novel sliding-mode-inspired structure, while the INDI analysis distinguishes itself from prior work by not relying on Jacobian linearization or perturbation boundedness assumptions.

5) Real-vehicle experimental validation on a Chevrolet Bolt EUV, demonstrating real-time implementability and practical path-tracking performance of the proposed controllers.

\subsection{Organization of the Paper}
The rest of this article is organized as follows. Section~\ref{sec2} presents the tracking error model with varying longitudinal motion. Section~\ref{sec3} presents the longitudinal-motion-aware feedback linearization and analyzes the stability of the internal dynamics. Section~\ref{sec4} presents the robust nonlinear controllers and their ultimate boundedness analysis. Section~\ref{sec5} reports simulation results, and Section~\ref{sec veh} presents real-vehicle experimental validation. Conclusions are given in Section~\ref{sec6}.

\section{Bicycle Model and Tracking Error Model}
\label{sec2}

This section first revisits the standard bicycle model used for lateral vehicle dynamics and then develops a control-oriented tracking-error model that explicitly accounts for variations in longitudinal speed and acceleration, which are often neglected under the widely used constant-speed assumption.
\subsection{Bicycle model}
This subsection briefly reviews the standard bicycle model and introduces the notation used. The bicycle model is a widely used abstraction that combines each axle into a single wheel, as depicted in Fig.~\ref{fig: bicycle_model}. Two reference frames are employed: an inertial frame and a body-fixed frame attached to the center of gravity (C.G.) at point $C$. We denote the unit vectors along the axes of the inertia and body frames as $I,J$ and $i,j$, respectively. Note that we depict $i$ and $j$ with a slight offset for visibility. The vehicle is represented as a rigid stick, with each axle pair simplified to a single tire. $\theta$ represents the heading angle of the vehicle and $\delta_f$ denotes the front-steering angle. $a$ and $b$ represent the distances from the C.G. to the front and rear axles, respectively.

\begin{figure}[htb!]
    \centering
    \includegraphics[width = 0.35\linewidth]{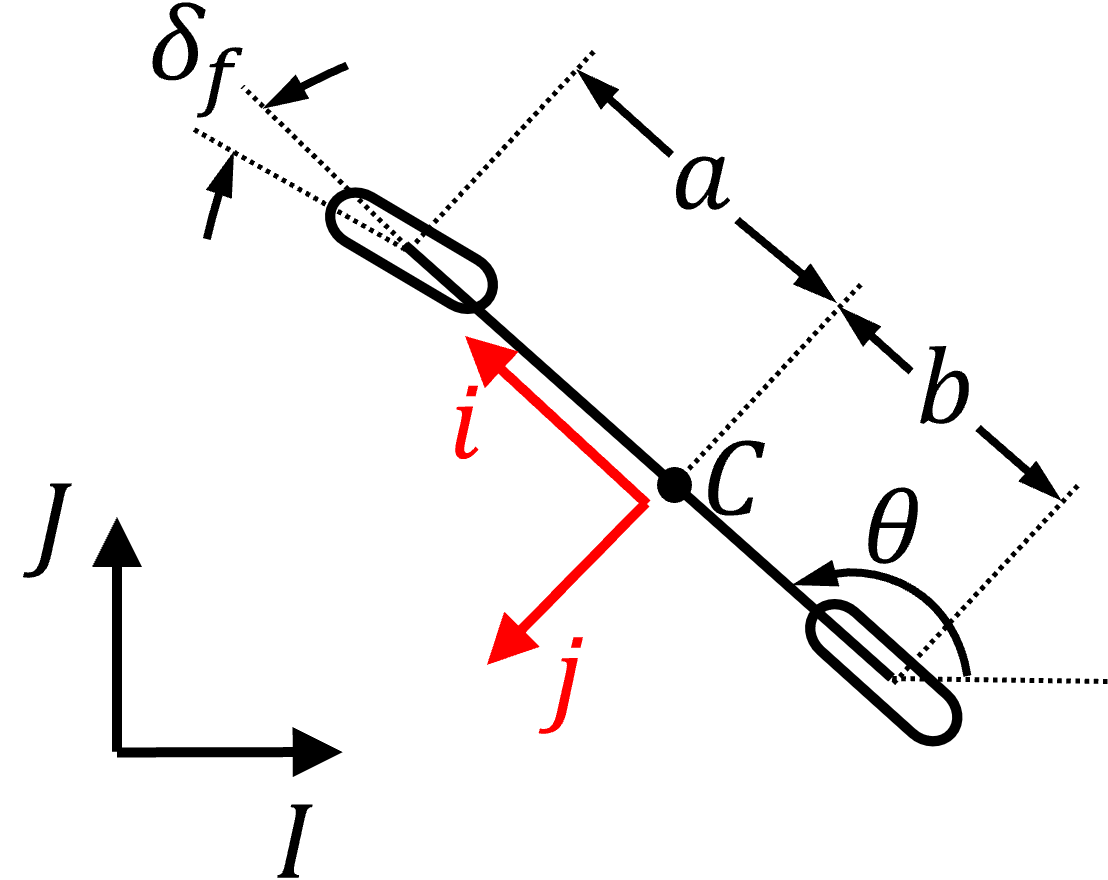}
    \caption{Bicycle model}
    \label{fig: bicycle_model}
\end{figure}

From Fig.~\ref{fig: bicycle_model}, the kinematic relationships are obtained:
\begin{align}
    i &= \cos{\theta} I + \sin{\theta} J, \quad j = -\sin{\theta} I + \cos{\theta} J, \label{eq: relationship_i_j} \\
    \tfrac{d i}{dt} &= \dot{\theta} \left(-\sin{\theta} I + \cos{\theta} J \right) = \dot{\theta} j, \quad \tfrac{d j}{dt} = -\dot{\theta} i. 
\end{align}
The vehicle velocity vector in the body frame is $\mathbf{v} = v_x i + v_y j,$ where $v_x$ and $v_y$ are the longitudinal and lateral speeds, respectively. The corresponding acceleration vector is:
\begin{align} \label{eq: expression_acceleration}
    \tfrac{d \mathbf{v}}{dt} &= \dot{v}_xi+v_x \tfrac{di}{dt} + \dot{v}_y j + v_y \tfrac{dj}{dt} = (\dot{v}_x- v_y \dot{\theta}) i + (\dot{v}_y + v_x \dot{\theta}) j.
\end{align}

Let $m$ and $I_z$ denote the vehicle mass and yaw moment of inertia, respectively. Let $C_f, C_r$ be the cornering stiffness of the front and rear tires. For small slip angles \cite{rajamani2011vehicle}, the lateral forces at the front ($F_f$) and rear ($F_r$) tires are given by:
\begin{align}
    F_f = C_f \alpha_f, \quad F_r = C_r \alpha_r,
\end{align}
where the slip angles, $\alpha_f$ and $\alpha_r$, are given by \cite{rajamani2011vehicle}
\begin{align}
    \alpha_f = \delta_f - \tfrac{v_y + a \dot{\theta}}{v_x}, \quad
    \alpha_r =  -\tfrac{v_y - b \dot{\theta}}{v_x}.
\end{align}

Using the lateral acceleration term from (\ref{eq: expression_acceleration}) and applying force and moment balance yields the  equations of motion:
\begin{equation}
\label{eq: equation_of_motion_1}
    m (\dot{v}_y + v_x \dot{\theta}) = F_f + F_r = C_f \delta_f  - \tfrac{C_f + C_r}{v_x} v_y - \tfrac{a C_f - b C_r}{v_x} \dot{\theta},
    \end{equation}
    \begin{equation}
\label{eq: equation_of_motion_2}
    I_z \Ddot{\theta} = a F_f - b F_r = a C_f \delta_f - \tfrac{a C_f - b C_r}{v_x} v_y - \tfrac{a^2 C_f + b^2 C_r}{v_x} \dot{\theta}.
\end{equation}

\subsection{Control-oriented tracking error dynamics model}

Building on the standard bicycle model, this subsection develops the control-oriented tracking-error model used in this work, distinguished from conventional forms by explicitly preserving variations in longitudinal speed and acceleration. For feedback controller design, expressing the motion equations in terms of position and heading errors relative to a desired path is beneficial. As depicted in Fig.~\ref{fig: representation_heading_position_errors}, denote $(X_v, Y_v)$ as the vehicle's C.G. in the inertial frame, and $(X_0, Y_0)$ as its projection on the reference path such that the vector from $(X_0, Y_0)$ to $(X_v, Y_v)$ is orthogonal to the path tangent. Define $e_{lat}$ as the distance between these points, with $e_t$ and $e_n$ being the tangential and normal unit vectors at $(X_0, Y_0)$. Let $\theta_R$ represent the path heading at $(X_0, Y_0)$. From kinematics, $\dot{\theta}_R = v_x / R$, where $R$ is the instantaneous radius of curvature. For a straight path, $R = \infty$, resulting in $\dot{\theta}_R = 0$.

\begin{figure}[htb!]
    \centering
    \includegraphics[width = 0.55\linewidth]{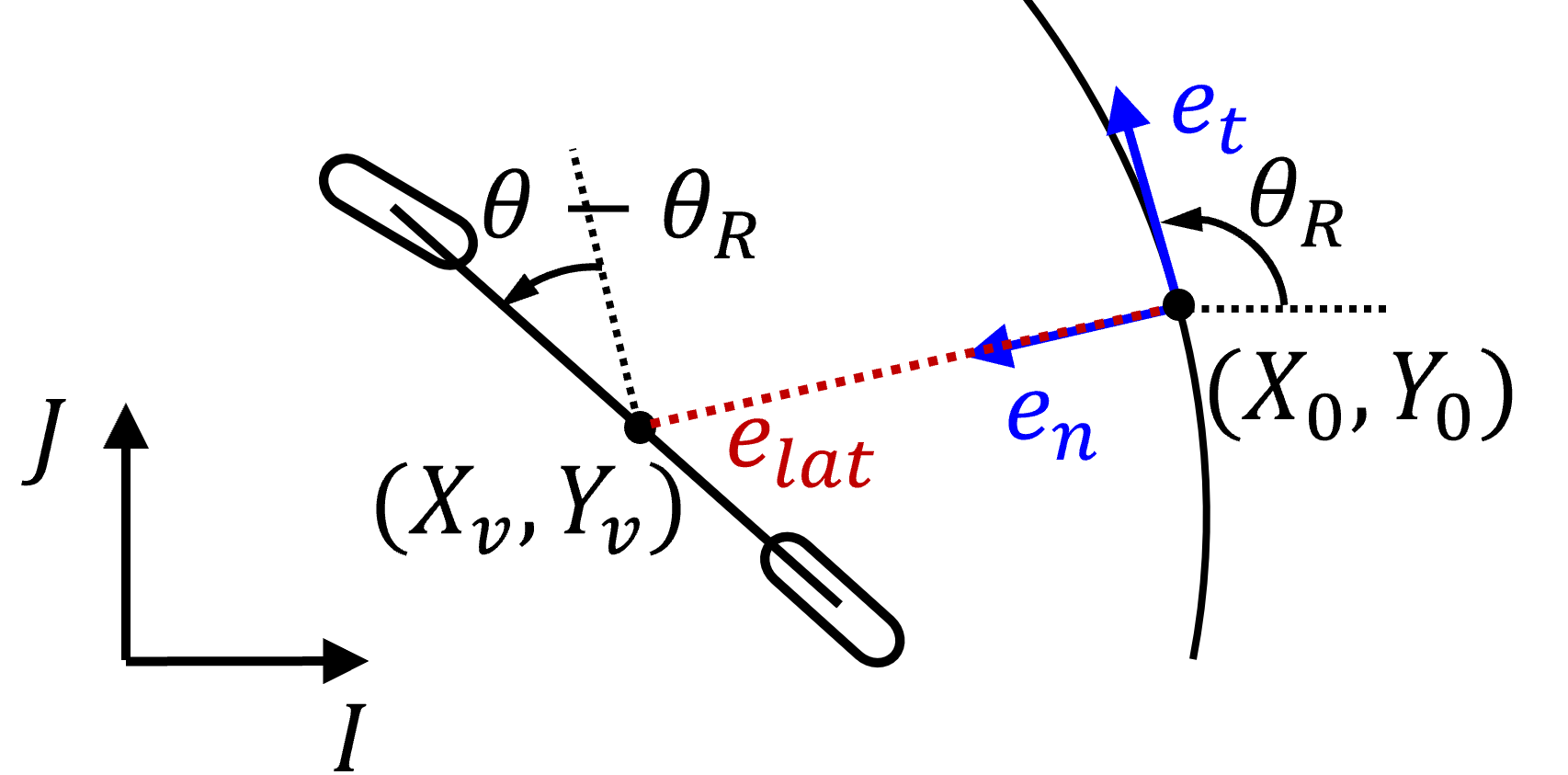}
    \caption{Representation of heading and position errors}
    \label{fig: representation_heading_position_errors}
\end{figure}

The coordinate transformation relating $e_t, e_n$ to $I, J$ is:
\begin{align} \label{eq: expression_et_en}
    e_t &= \cos \theta_R I + \sin \theta_R J, \quad e_n = -\sin \theta_R I + \cos \theta_R J. 
\end{align}
By applying the association of $i, j$ with $I, J$ in \eqref{eq: relationship_i_j}, one can derive the expression for $i, j$ in terms of $e_t, e_n$ as 
\begin{align}
    i &= \cos(\theta-\theta_R)e_t + \sin (\theta-\theta_R)e_n, \\
    j &=-\sin (\theta-\theta_R) e_t + \cos (\theta - \theta_R)e_n.
\end{align}

To obtain a control-oriented tracking-error model for path tracking, we adopt the following standard approximations widely used in the literature, e.g., \cite{rajamani2011vehicle,liu2020lateral}: (i) the heading error is assumed small, so $\sin(\theta-\theta_R)\approx\theta-\theta_R$ and $\cos(\theta-\theta_R)\approx1$; (ii) $\dot R$ is neglected in the  model; and (iii) quadratic and higher-order state terms are discarded. These assumptions are standard in vehicle path-tracking control and have been experimentally validated in previous work \cite{liu2020lateral}.

Lateral error rate correlates with the vehicle's speed component along the $e_n$ direction. Consequently, 
\begin{align} \label{eq: expression_elat_dot}
    {\dot e}_{lat} &= (v_x i + v_y j) \cdot e_n = v_x \sin (\theta-\theta_R) + v_y \cos (\theta-\theta_R) \notag\\
    &\qquad\qquad\qquad\qquad\approx v_x (\theta-\theta_R) + v_y, \\
    {\ddot e}_{lat} &= \tfrac{d}{dt} \left[\left(v_x i + v_y j \right) \cdot e_n \right] \notag\\
    &= \tfrac{d}{dt} \left(v_x i + v_y j \right) \cdot e_n + \left(v_x i + v_y j \right) \cdot \tfrac{d e_n}{d t} \notag\\
    &=  \left(\left(\dot{v}_x- v_y \dot{\theta}\right) i + \left(\dot{v}_y + v_x \dot{\theta} \right) j \right) \cdot e_n \notag\\
    & \quad\, + (v_x i + v_y j) \cdot (-{\dot \theta}_R e_t),
\end{align}
where the acceleration formula from \eqref{eq: expression_acceleration} and the derivative of $e_n$ from \eqref{eq: expression_et_en} were utilized. Noting that $j \cdot e_n = \cos (\theta - \theta_R),$ $i \cdot e_n = \sin (\theta-\theta_R),$ $i \cdot e_t = \cos (\theta-\theta_R),$ and $j \cdot e_t = -\sin (\theta-\theta_R)$, we obtain 
\begin{align} \label{eq: expression_elat_ddot}
    {\ddot e}_{lat} &= \left(\dot{v}_x- v_y \dot{\theta}\right) \sin{(\theta - \theta_R)}+\left(v_x \dot \theta + \dot v_y\right) \cos{(\theta - \theta_R)}  \notag\\
    & \quad\, - \left(v_x {\dot \theta}_R \cos (\theta-\theta_R) -v_y {\dot \theta}_R \sin (\theta-\theta_R) \right), \notag\\
    &\approx \dot v_x({\theta} - {\theta}_R)+v_x ({\dot\theta} - {\dot\theta}_R)+\dot v_y.
\end{align}
Define the heading error as ${\tilde \theta} := \theta - {\theta}_R$. Observe that $v_y$ and ${\dot v_y}$ can be expressed through \eqref{eq: expression_elat_dot} and \eqref{eq: expression_elat_ddot}, respectively, as 
\begin{align*}
    v_y &= {\dot e}_{lat} - v_x {\tilde \theta}, \quad
    \dot v_y = {\ddot e}_{lat} - v_x \dot {\tilde \theta}-\dot v_x {\tilde \theta}. 
\end{align*}
The equations of motion presented in \eqref{eq: equation_of_motion_1} and \eqref{eq: equation_of_motion_2} can be reformulated using the state variables, $e_{lat}$, $\dot e_{lat}$, $\tilde \theta$, $\dot{\tilde \theta}$ as
\begin{align}
     \Ddot{e}_{lat} &=  \underbrace{\left(\tfrac{C_f + C_r}{m}+\dot v_x \right) \tilde{\theta} -\tfrac{C_f + C_r}{mv_x} \dot{e}_{lat} - \tfrac{a C_f - b C_r}{mv_x}  \dot{\tilde \theta}}_{f_1} \notag\\
    & \quad\, + \underbrace{\tfrac{C_f}{m}}_{g_1} \delta_f \underbrace{- \left( v_x + \tfrac{a C_f - b C_r}{mv_x} \right)}_{d_1} \dot{\theta}_R, \label{eq: e_ddot_govern}
    \end{align}
    \begin{align}
    \Ddot{\tilde \theta} &= \underbrace{\tfrac{a C_f - b C_r}{I_z}  \tilde{\theta}-\tfrac{a C_f - b C_r}{I_zv_x}\dot{e}_{lat} - \tfrac{a^2 C_f + b^2 C_r}{I_zv_x}  \dot{\tilde \theta}}_{f_2} \notag\\
    & \quad\,  + \underbrace{\tfrac{a C_f}{I_z}}_{g_2} \delta_f  \underbrace{-\tfrac{a^2 C_f + b^2 C_r+I_z\dot v_x}{I_zv_x}}_{d_2} \dot{\theta}_R\label{eq: theta_ddot_govern}.
\end{align}
Here, $\dot{\theta}_R = \tfrac{v_x}{R}$ and $\ddot{\theta}_R = \tfrac{\dot v_x}{R}+v_x \tfrac{d}{dt} \left(\tfrac{1}{R} \right) = \tfrac{\dot v_x}{R}=\tfrac{\dot v_x}{v_x}\dot{\theta}_R$ were used. It should be noted that the expression for $\ddot{\theta}_R$ was obtained since $\dot R$ is omitted.
\begin{rem}
In the literature, terms involving $\dot v_x$ do not appear because longitudinal speed is assumed constant \cite{rajamani2011vehicle,liu2020lateral}. In contrast, such terms are retained in the present work because varying longitudinal motion is explicitly considered.
\end{rem}

The state-space representation of (\ref{eq: e_ddot_govern})-(\ref{eq: theta_ddot_govern}) is expressed as
\begin{equation}
    \mathbf{\dot x}=\mathbf{f}+\mathbf{g} \delta_f+\mathbf{d}\dot{\theta}_R, \label{eq: state_stpace_error}
\end{equation}
where $\mathbf{x}=[e_{lat},\dot e_{lat},\tilde \theta,\dot{\tilde \theta}]^T$, $\mathbf{f}=[\dot e_{lat},f_1,\dot{\tilde \theta},f_2]^T$, $\mathbf{g}=[0, g_1, 0, g_2]^T$, $\mathbf{d}=[0,d_1,0,d_2]^T$.

\section{Longitudinal-Motion-Aware Feedback Linearization and Internal Dynamics Analysis}
\label{sec3}
This section derives a longitudinal-motion-aware lateral control law using feedback linearization (see, e.g., Section 13 of \cite{khalil2002nonlinear}). We linearize the input–output map from steering angle \(\delta_f\) to lateral error \(e_{lat}\), explicitly incorporating longitudinal speed and acceleration into the control law while making the resulting external dynamics independent of them. We then analyze the stability of the associated internal dynamics, which become unobservable under input–output linearization. Parametric uncertainty is treated in the next section.

The following assumptions, which are common for controller design and stability analysis \cite{peng1990lateral,liu2020robust}, are introduced to facilitate the subsequent development and theoretical results.
\begin{assump}
\label{assump: parameters}
Each parameter \(z \in \{C_f, C_r, m, I_z\}\) is treated as an uncertain constant over the time scale of interest, with known bounds \(0 < z_{\min} \le z \le z_{\max}\). Moreover, each \(z\) has a nominal value \(\hat z\), determined via offline calibration or online identification. The geometric parameters \(a\) and \(b\) are taken as fixed, known, positive constants \cite{tagne2015design,shirazi2017mathcal}.
\end{assump}

\begin{rem}
When load variation is small relative to the nominal mass (e.g., passenger weight in a car), the resulting C.G. shift has minimal impact \cite{darbha2025robust}. For systems with large load-dependent C.G. changes (e.g., trucks with heavy cargo), one can instead use a model that explicitly accounts for it \cite{darbha2025robust} or update \(a\) and \(b\) via additional sensing or estimation.
\end{rem}

\begin{assump} 
\label{assump: longitudinal}
The measurable longitudinal speed and acceleration are assumed to be differentiable and bounded by $v_{x,min} \leq v_x(t) \leq v_{x,max}$ and $\dot{v}_{x,min} \leq \dot{v}_x(t) \leq \dot{v}_{x,max}$ for all $t\in\mathbb{R}$. As only forward motion is considered, $v_{x,\min} > 0$. Also, $\dot{v}_{x,\min} < 0<\dot{v}_{x,\max}$ is assumed to account for deceleration and acceleration. The path radius is assumed to satisfy $0<|R|_{min}\leq |R(t)|\leq\infty$ for all $t\in\mathbb{R}$. Consequently, $\dot\theta_R(t)=\tfrac{v_x(t)}{R(t)}$ satisfies $\dot\theta_{R,min}:=-\tfrac{v_{x,max}}{|R|_{min}}\leq\dot\theta_R(t)\leq\dot\theta_{R,max}:=\tfrac{v_{x,max}}{|R|_{min}}$ for all $t\in\mathbb{R}$.

\end{assump}

Following the standard process of feedback linearization, we proceed by taking the derivatives of the output $e_{lat}$ \cite{khalil2002nonlinear}:
\begin{align}
\label{eq: output_derivative_1}
    \tfrac{de_{lat}}{dt}&=\dot e_{lat}, \\
\label{eq: output_derivative_2}
    \tfrac{d^2 e_{lat}}{dt^2} &= \ddot e_{lat} =f_1+g_1\delta_f+d_1\dot{\theta}_R.
\end{align}
Since the control input $\delta_f$ appears explicitly in $\ddot e_{lat}$, the system (\ref{eq: state_stpace_error}) with $e_{lat}$ as output has a relative degree of 2. The external states can be selected as 
\begin{equation}
\label{eq: xi_def}
\boldsymbol{\xi}=[\xi_1, \xi_2]^T:=[e_{lat}, \dot e_{lat}]^T.
\end{equation}
If all parameters were precisely known, using
\begin{equation}
\label{eq: nonlinearity_cancellation_exact}
    \delta_f=\tfrac{1}{g_1}\left(-f_1-d_1\dot{\theta}_R-\mathbf{K}\boldsymbol{\xi}\right),
\end{equation}
where $\mathbf{K}=[k_1, k_2]$, the dynamics of the external states in \eqref{eq: output_derivative_1}-\eqref{eq: output_derivative_2} could be exactly linearized to
\begin{equation}
\label{eq: linearized_external_exact}
    \boldsymbol{\dot\xi}={\small\begin{bmatrix}
0 & 1 \\
-k_1 & -k_2
\end{bmatrix}}\boldsymbol{\xi}.
\end{equation}
\begin{rem}
From the expressions for \(f_1\), \(g_1\), and \(d_1\) in \eqref{eq: e_ddot_govern}, the feedback law in \eqref{eq: nonlinearity_cancellation_exact} explicitly incorporates longitudinal speed and acceleration, thereby making the controller longitudinal-motion-aware. Meanwhile, the external dynamics in \eqref{eq: linearized_external_exact} no longer depend on these quantities, allowing for more consistent tracking performance under varying longitudinal motion.
\end{rem}

Stabilizing values of \(k_1\) and \(k_2\) for the external dynamics are readily obtained by choosing \(k_1>0\) and \(k_2>0\), since the characteristic polynomial of \eqref{eq: linearized_external_exact} is \(s^2+k_2 s+k_1\), which is Hurwitz under these conditions. The impact of parametric uncertainties on the external dynamics will be addressed in the upcoming section. Presently, we aim to find the internal state in order to build a diffeomorphism for a useful coordinate transformation (see, for example, page 508 of \cite{khalil2002nonlinear}).

The internal state, $\boldsymbol{\eta}=[{\eta}_1, {\eta}_2]^T$, is chosen by solving
\begin{equation}
    \tfrac{\partial{\eta}_i}{\partial\mathbf{x}} \mathbf{g}=\tfrac{C_f}{m}\tfrac{\partial{\eta}_i}{\partial \dot e_{lat}}+\tfrac{aC_f}{I_z}\tfrac{\partial{\eta}_i}{\partial\dot{\tilde\theta}}=0, ~\text{for}~ i=1,2.
\end{equation}
An obvious solution to the above equation is
\begin{equation}
    \boldsymbol{\eta}=[{\eta}_1, {\eta}_2]^T:=\left[\tilde\theta, ma\dot e_{lat}-I_z\dot{\tilde\theta}\right]^T. \label{eq: internal state}
    \end{equation}
     The coordinate transformation 
     \begin{equation}
     \label{eq: coord_trans}
     \mathbf{T(\mathbf{x})}=[\boldsymbol{\xi}^T(\mathbf{x}), \boldsymbol{\eta}^T(\mathbf{x})]^T=\mathbf{M}\mathbf{x}
     \end{equation}
     is a linear operator with \begin{equation}
    \mathbf{M}={\small\begin{bmatrix}
        1&0&0&0\\
        0&1&0&0\\
        0&0&1&0\\
        0&ma&0&-I_z\\
    \end{bmatrix}}.
    \end{equation}
    Noting that $m,a,I_z$ are positive due to their physical meanings, it can be seen that (1) $\mathbf{M}$ is full rank for all $\mathbf{x}\in\mathbb{R}^4$. Furthermore, (2) $\displaystyle\lim_{\|x\|_2\to\infty}\|\mathbf{T}(\mathbf{x})\|_2\geq\lim_{\|x\|_2\to\infty}\sigma_{min}(\mathbf{M})\|x\|_2=\infty$, where $\sigma_{min}(\mathbf{M})$ is the smallest singular value of $\mathbf{M}$ and is nonzero since $\mathbf{M}$ is full rank. These conditions indicate that $\mathbf{T}(\mathbf{x})$ is a global diffeomorphism \cite{khalil2002nonlinear}.

We can now analyze the internal dynamics. By (\ref{eq: internal state}),
\begin{equation}
    \boldsymbol{\dot\eta}={\small\begin{bmatrix}
        \dot{\tilde\theta}\\
        maf_1-I_zf_2+(mad_1-I_zd_2)\dot{\theta}_R
    \end{bmatrix}}.
\end{equation}
By rearranging the equation above, we can rewrite it in terms of the transformed state variables $\boldsymbol{\xi}$ and $\boldsymbol{\eta}$, leading to
\begin{align}
\label{eq: internal dynamics}
\boldsymbol{\dot\eta}=&f_{\eta}(t,\boldsymbol{\eta},\boldsymbol{\xi},\dot{\theta}_R)\notag\\
=&\underbrace{\small\begin{bmatrix}
        0&-\tfrac{1}{I_z}\\
         (a+b)C_r+a m \dot v_x&-\tfrac{(a+b) bCr }{v_x I_z}
    \end{bmatrix}}_{\mathbf{A_\eta}(t)}\boldsymbol{\eta} \notag\\
    &+\underbrace{\small\begin{bmatrix}
        0&\tfrac{a m}{I_z}&0\\
         0&\tfrac{ (a+b) (a b m-I_z)C_r}{v_x I_z}&\tfrac{(a b +b^2 )C_r-a m v_x^2+I_z \dot v_x}{v_x}
    \end{bmatrix}}_{\mathbf{B}_\eta(t)}
    \begin{bmatrix}
    \boldsymbol{\xi}\\    
    \dot{\theta}_R
    \end{bmatrix},
\end{align}
wherein $\boldsymbol{\xi}$ and $\dot{\theta}_R$ are treated as inputs to the internal dynamics. We begin by analyzing the stability of the unforced dynamics $f_\eta(t,\boldsymbol{\eta},0,0)$, which extends the concept of zero dynamics \cite{khalil2002nonlinear}. Since $v_x$ and $\dot v_x$ vary, $\boldsymbol{\dot\eta}=f_\eta(t,\boldsymbol{\eta},0,0)=\mathbf{A_\eta}(t)\boldsymbol{\eta}$ forms an LTV system. A key result from \cite{ioannou1996robust} is utilized to establish the global exponential stability of this LTV system.

\begin{lem}(Theorem 3.4.11 in \cite{ioannou1996robust})
\label{lem: LTV_stability}
Let the elements of $\mathbf{A}(t)$ in $\mathbf{\dot z}=\mathbf{A}(t)\mathbf{z}$ be differentiable and bounded functions of time and assume that 
\begin{itemize}
\item For some constant $\sigma_s>0$, all eigenvalues of $\mathbf{A}(t)$ have real parts less than or equal to $-\sigma_s$ $\forall t\geq0$.
\end{itemize}
\begin{itemize}
  \item[(i)] If for some finite $\rho>0$, $\int_0^\infty \| {\mathbf{\dot A}}(\tau)\|^2d\tau\leq\rho$, then the equilibrium state $\boldsymbol{\eta} = 0$ is globally exponentially stable.
  \item[(ii)] If any one of the following conditions holds:
  \begin{itemize}
    \item[(a)] $ \int_t^{t+T} \|{\mathbf{\dot A}}(\tau)\| d\tau \leq \mu T + \alpha_0,$
    \item[(b)] $ \int_t^{t+T} \|{\mathbf{\dot A}}(\tau)\|^2 d\tau \leq \mu^2 T + \alpha_0,$
    \item[(c)] $ \|{\mathbf{\dot A}}(t)\| \leq \mu, $
  \end{itemize}
\end{itemize}
for some $\alpha_0, \mu \in \mathbb{R}^+$ and $\forall t \geq 0, T \geq 0$, then there exists a $\mu^* > 0$ such that if $\mu \in [0, \mu^*)$, the equilibrium state $\mathbf{z}=0$ is globally exponentially stable.
\end{lem}

We can establish the conditions of global exponential stability for the unforced internal dynamics using Lemma \ref{lem: LTV_stability}.
\begin{prop}
\label{prop: internal GES}
    For the unforced internal dynamics $\boldsymbol{\dot\eta}=f_\eta(t,\boldsymbol{\eta},0,0)=\mathbf{A_\eta}(t)\boldsymbol{\eta}$, assume that: 
    
    \begin{itemize}
    \item[(A1)] $(a+b)C_r>am|\dot v_{x,min}|$.
    \end{itemize}
    Then,
    \begin{itemize}
    \item[(i)] If for some finite $\rho>0$, $\int_0^\infty \left(\dot v^2_x(t)+\ddot v^2_x(t)\right)d\tau\leq\rho$, then the equilibrium state $\boldsymbol{\eta} = 0$ is globally exponentially stable.
  \item[(ii)] If any one of the following conditions holds:
  \begin{itemize}
    \item[(a)] $ \int_t^{t+T} \left(\left|\dot v_x(\tau)\right|+\left|\ddot v_x(\tau)\right|\right) d\tau \leq \mu T + \alpha_0,$
    \item[(b)] $ \int_t^{t+T} \left(\dot v^2_x(\tau)+\ddot v^2_x(\tau)\right) d\tau \leq \mu^2 T + \alpha_0,$
    \item[(c)] $ \left|\dot v_x(t)\right|+\left|\ddot v_x(t)\right| \leq \mu, $
  \end{itemize}
\end{itemize}
for some $\alpha_0, \mu \in \mathbb{R}^+$ and $\forall t \geq 0, T \geq 0$, then there exists a $\mu^* > 0$ such that if $\mu \in [0, \mu^*)$, the equilibrium state $\boldsymbol{\eta}=0$ is globally exponentially stable.
\end{prop}
\begin{proof}
   From (\ref{eq: internal dynamics}), each entry of $\mathbf{A_\eta}(t)$ is differentiable and bounded because $\tfrac{1}{v_x}$ and $\dot v_x$ are differentiable and bounded.
    
     At every $t$, the characteristic equation of $\mathbf{A_\eta}(t)$ is 
    \begin{equation}
        \det(\mathbf{A_\eta}(t)-\lambda\mathbf{I})=\lambda^2+\tfrac{(a+b) bCr }{v_x(t) I_z}\lambda +\tfrac{(a+b)C_r+a m \dot v_x(t)}{I_z}.
    \end{equation}
    By the Routh-Hurwitz criterion, the solutions to the above equation all have negative real parts if $\tfrac{(a+b) bCr }{v_x(t) I_z}>0$ and $\tfrac{(a+b)C_r+a m \dot v_x(t)}{I_z}>0$. Referring to Assumptions \ref{assump: parameters} and \ref{assump: longitudinal}, the two inequalities are satisfied if $(a+b)C_r>-am\dot v_x(t)\geq am|\dot v_{x,min}|$. Therefore, if (A1) is satisfied, all eigenvalues of $\mathbf{A_\eta}(t)$ have negative real parts $\forall t\geq0$.

    Note that from (\ref{eq: internal dynamics}), $\mathbf{A_\eta}(t)$ can be written as $\mathbf{A_\eta}(t)=\mathbf{A_1}+\tfrac{1}{v_x(t)}\mathbf{A_2}+\dot v_x(t)\mathbf{A_3}$, hence
    \begin{align}
        \mathbf{\dot A_\eta}(t)=-\tfrac{\dot v_x(t)}{v_x^2(t)}\mathbf{A_2}+\ddot v_x(t)\mathbf{A_3},
    \end{align}
    where $\mathbf{A_2}=
      {\small\begin{bmatrix}
        0&0\\
         0&-\tfrac{(a+b) bCr }{ I_z}
    \end{bmatrix}},
        \mathbf{A_3}=
        {\small\begin{bmatrix}
        0&0\\
         a m &0
    \end{bmatrix}}.$
    
    For conditions (ii)(a) and (ii)(c), we have
    \begin{align}
        \|\mathbf{\dot A_\eta}(t)\|
        &\leq \tfrac{\left|\dot v_x(t)\right|}{v_{x,min}^2} \|\mathbf{A_2}\|+|\ddot v_x(t)| \|\mathbf{A_3}\| \notag \\
        &\leq \max\left(\tfrac{\|\mathbf{A_2}\|}{v_{x,min}^2},\|\mathbf{A_3}\|\right) \left(\left|\dot v_x(t)\right|+\left|\ddot v_x(t)\right|\right),
    \end{align}
    conditions (ii)(a) and (ii)(c) follow from Lemma \ref{lem: LTV_stability}.

    For conditions (i) and (ii)(b), we utilize the induced 2-norm:
    \begin{align}
        \|\mathbf{\dot A_\eta}(t)\|_2^2&\leq \left(\|-\tfrac{\dot v_x(t)}{v_x^2(t)}\mathbf{A_2}\|_2+\|\ddot v_x(t)\mathbf{A_3}\|_2\right)^2\notag\\
        &\leq2\|-\tfrac{\dot v_x(t)}{v_x^2(t)}\mathbf{A_2}\|_2^2+2\|\ddot v_x(t)\mathbf{A_3}\|_2^2  \notag \\
        &\leq2 \underbrace{\left(\tfrac{(a+b) bCr }{ I_zv_{x,min}^2}\right)^2}_{a_1}\dot v^2_x(t) + 2a^2m^2\ddot v^2_x(t) \notag \\
        &\leq 2\max\left(a_1,a^2m^2\right) \left(\dot v^2_x(t)+\ddot v^2_x(t)\right),
    \end{align}
    using Lemma \ref{lem: LTV_stability}, conditions (i) and (ii)(b) are established.
\end{proof}
\begin{rem}
\label{rem: internal_state_parameter_condition}
    Assumption (A1) in Proposition \ref{prop: internal GES} is trivially satisfied in practical scenarios. For a typical passenger vehicle \cite{liu2020robust}, with parameters $a\approx1.3m$, $b\approx1.6m$, $C_r\approx380000N/rad$, $m\approx1900kg$, a deceleration of $450m/s^2$ would be required for (A1) to be violated.
\end{rem}
    \begin{rem}
Note that any one of the conditions in Proposition \ref{prop: internal GES} is sufficient. Roughly speaking, condition (i) requires the longitudinal acceleration and jerk to remain square-integrable over time, whereas conditions (ii)(a)--(c) require their accumulated size over finite time intervals, or their magnitude, to remain uniformly controlled. These conditions are consistent with the objectives of the longitudinal controller and planner, such as maintaining a desired speed \cite{swaroop1994string} and avoiding unnecessarily aggressive acceleration and jerk \cite{li2024sequencing,werling2012optimal}. They do not exclude rapid braking or acceleration, so long as the resulting signals remain finite in the classical sense.
\end{rem}

Proposition \ref{prop: internal GES} establishes global exponential stability of $\boldsymbol{\eta}$ when $\boldsymbol{\xi}=\dot{\tilde\theta} = 0$. In practice, $\boldsymbol{\xi}$ may not fully converge because of uncertainties and disturbances, and $\dot{\tilde\theta} \neq 0$ unless the path is straight. Hence, only boundedness of $\boldsymbol{\eta}$ can be guaranteed. To formalize this, we invoke the following proposition.

\begin{prop}
\label{prop: internal_ISS}
    If the unforced internal dynamics $\dot{\boldsymbol{\eta}} = f_\eta(t, \boldsymbol{\eta}, 0, 0)$ is globally exponentially stable, then the internal dynamics $\dot{\boldsymbol{\eta}} = f_\eta(t, \boldsymbol{\eta}, \boldsymbol{\xi}, \dot{\theta}_R)$ is input-to-state stable (ISS) with respect to the inputs $(\boldsymbol{\xi}, \dot{\theta}_R)$.
\end{prop}

\begin{proof}
    From (\ref{eq: internal dynamics}), we have
    \begin{align}
        \label{eq: f_eta_norm}
        \left\|\tfrac{\partial f_\eta}{\partial (\boldsymbol{\eta}, \boldsymbol{\xi}, \dot{\theta}_R)}(t)\right\| = \left\|[\mathbf{A}_\eta(t) \quad \mathbf{B}_\eta(t)]\right\|.
    \end{align}
    Since $v_x(t)$, $\tfrac{1}{v_x(t)}$, and $\dot{v}_x(t)$ are bounded, the right-hand side of (\ref{eq: f_eta_norm}) is uniformly bounded for all $t \geq 0$ and $(\boldsymbol{\eta}, \boldsymbol{\xi}, \dot{\theta}_R) \in \mathbb{R}^5$. That is, there exists a constant $L > 0$ such that
    \begin{equation}
        \left\|\tfrac{\partial f_\eta}{\partial (\boldsymbol{\eta}, \boldsymbol{\xi}, \dot{\tilde\theta}_R)}(t)\right \| \leq L, \quad \forall (t, (\boldsymbol{\eta}, \boldsymbol{\xi}, \dot{\tilde\theta}_R)) \in [0, \infty) \times \mathbb{R}^5.
    \end{equation}
    This implies that $f_\eta$ is globally Lipschitz in $(\boldsymbol{\eta}, \boldsymbol{\xi}, \dot{\theta}_R)$, uniformly  $t$. Combined with the global exponential stability of the unforced internal dynamics, Lemma 4.6 in \cite{khalil2002nonlinear} then implies that the internal dynamics $\dot{\boldsymbol{\eta}} = f_\eta(t, \boldsymbol{\eta}, \boldsymbol{\xi}, \dot{\theta}_R)$ is ISS.
\end{proof}

Proposition \ref{prop: internal_ISS} ensures that, under the conditions specified in Proposition \ref{prop: internal GES}, the internal state $\boldsymbol{\eta}(t)$ remains bounded for any bounded input $(\boldsymbol{\xi}(t),\dot{\theta}_R(t))$. Moreover, as $t$ increases, $\boldsymbol{\eta}(t)$ will be ultimately bounded by a class $\mathcal{K}$ function of $\sup_{t\geq 0}\|(\boldsymbol{\xi}(t),\dot{\theta}_R(t))\|$ \cite{khalil2002nonlinear}.

\section{Robust Nonlinear Controllers Design}
\label{sec4}
As discussed in the previous section, when all parameters are known exactly, using the feedback law in (\ref{eq: nonlinearity_cancellation_exact}) converts the external state dynamics to (\ref{eq: linearized_external_exact}), which can be readily stabilized with suitable $k_1$ and $k_2$. In reality, however, the system parameters are uncertain: $m$ and $I_z$ may change with vehicle loading, $C_f$ and $C_r$ depend on road conditions, tire wear, and temperature. We therefore introduce two robustification methods, LR and INDI, to handle these uncertainties.

The proposition of two alternative methods is motivated by their distinct trade-offs in practical implementation. LR requires the nominal values and explicit bounds of all parameters; INDI only assumes parameter boundedness and relies on the nominal values of two parameters, but uses additional measurements. This offers practitioners flexibility based on available parameter information and sensing resources.

Note that the boundedness of the internal states has already been established previously with minimal parameter assumptions (see Remark \ref{rem: internal_state_parameter_condition}), allowing us to concentrate on the robustness of the external dynamics. We characterize the robustness in the sense of ultimate boundedness \cite{khalil2002nonlinear}.

\subsection{Lyapunov Redesign}
The LR technique adjusts a pre-existing stabilizing control law of a nominal model by incorporating a correction term based on the Lyapunov function's derivative and a bound on the uncertainty. This adjustment keeps the Lyapunov function decreasing despite the presence of uncertainties.



Let $\hat f_1, \hat g_1,$ and $\hat d_1$ represent the nominal values of $f_1, g_1,$ and $d_1$, respectively, implying these are the values of $f_1, g_1,$ and $d_1$ when each parameter $z \in \{C_f, C_r, m, I_z\}$ is substituted by its nominal value $\hat z$ (refer to Assumption \ref{assump: parameters}).

For LR, Instead of the feedback law in (\ref{eq: nonlinearity_cancellation_exact}), we employ
\begin{equation}
\label{eq: LR_full_feedback}
    \delta_f=\tfrac{1}{\hat g_1}\left(-\hat f_1-\hat d_1\dot{\theta}_R-\mathbf{K}\boldsymbol{\xi}+u_r\right).
\end{equation}
 Substituting the above equation into (\ref{eq: output_derivative_1})-(\ref{eq: output_derivative_2}) and rearranging: 
\begin{equation}
\label{eq: external_LR}
\boldsymbol{\dot\xi}=\underbrace{{\small\begin{bmatrix}
        0 & 1\\
        -k_1 & -k_2
    \end{bmatrix}}}_{\mathbf{\hat A}}\boldsymbol{\xi}+\underbrace{{\small\begin{bmatrix}
        0\\1
    \end{bmatrix}}}_{\mathbf{\hat B}}\left(u_r+d_{lr}(\cdot)\right).
\end{equation}
Here, 
\begin{align}
\label{eq: perturbterm}
d_{lr}(\cdot)=&\left(\tfrac{g_1}{\hat g_1}-1\right)\left(-k_1e_{lat}-k_2\dot e_{lat}+u_r\right)+f_1-\tfrac{g\hat f_1}{\hat g_1} \notag\\
&+\left(d_1-\tfrac{g\hat d_1}{\hat g_1}\right)\dot\theta_R
\end{align}
is regarded as a perturbation term, $k_1$ and $k_2$ are selected such that $\mathbf{\hat A}$ is Hurwitz, $u_r$ is the correction term to counteract $d_{lr}$.

Since $\mathbf{\hat{A}}$ is Hurwitz, there exists a positive definite symmetric matrix $\mathbf{P}$ such that 
\begin{equation}
\label{eq: LR_Lyapunov_equation}
\mathbf{P}\mathbf{\hat A}+\mathbf{\hat A}^T\mathbf{P}=-\mathbf{I}.
\end{equation}
\begin{equation}
\label{eq: perturb_bound}
\text{Suppose  }\left|d_{lr}(\cdot)\right|\leq\rho\left(\mathbf{x},v_x,\dot v_x, \dot{\theta}_R\right)+\kappa_0|u_r|,\qquad 
\end{equation}
with $\rho(\cdot)\geq0$ being a continuous function and $0\leq\kappa_0<1$. 
\begin{rem}
    In (\ref{eq: perturb_bound}), the specific expression of $\rho(\cdot)$ and $\kappa_0$ can be derived by expanding $f_1$, $g_1$, $d_1$, $\hat f_1$, $\hat g_1$, and $\hat d_1$ in (\ref{eq: perturbterm}) using (\ref{eq: e_ddot_govern}), yielding
    \begin{align}
        \label{eq: purterbation bound}d_{lr}(\cdot)=&\alpha_d\cdot\left(-k_1e_{lat}-k_2\dot e_{lat}+u_r-\dot v_x\tilde\theta+v_x\dot\theta_R\right) \notag \\
        &+\gamma_d\cdot\left(\tilde\theta+\tfrac{b\left( \dot{\tilde\theta}+\dot\theta_R\right)-\dot e_{lat}}{v_x}\right), \\
        \left|d_{lr}(\cdot)\right|\leq&\max\left(|\alpha_d|\right)|u_r| +\max\left(|\gamma_d|\right)\left|\tilde\theta+\tfrac{b\left( \dot{\tilde\theta}+\dot\theta_R\right)-\dot e_{lat}}{v_x}\right|  \notag \\
        &+\max\left(|\alpha_d|\right) \left|-k_1e_{lat}-k_2\dot e_{lat}-\dot v_x\tilde\theta+v_x\dot\theta_R\right|,
    \end{align}
    where $\alpha_d=\tfrac{C_f\hat m}{\hat C_f m}-1$ and $\gamma_d=\tfrac{1}{m}\left(C_r-C_f\tfrac{\hat C_r}{\hat C_f}\right)$, their maximum absolute value can be calculated once the bounds and nominal value of $C_f, C_r$ and $m$ are specified. In the above inequality, the first term on the right side gives the expression of $\kappa_0|u_r|$, and the last two terms give the expression of $\rho(\cdot)$.
\end{rem}

A naive choice for the correction term is \cite{khalil2002nonlinear,estrada2024passive}
\begin{equation}
\label{eq: LR_discontinuous}
u_r=-\beta\left(\mathbf{x},v_x,\dot v_x, \dot{\theta}_R\right) \text{sgn}(w),
\end{equation}
where 
\begin{equation}
\beta\left(\mathbf{x},v_x,\dot v_x, \dot{\theta}_R\right)\geq\tfrac{\rho\left(\mathbf{x},v_x,\dot v_x, \dot{\theta}_R\right)}{1-\kappa_0},\quad
\label{eq: LR_w}
w=2\mathbf{\hat B}^T\mathbf{P}\boldsymbol{\xi}.
\end{equation}
\begin{rem}
Discontinuous control such as (\ref{eq: LR_discontinuous}) is known to cause chattering in practice \cite{khalil2002nonlinear}, leading to significant actuator wear. A well-known continuous LR control law was introduced in \cite{corless1981continuous} (see also \cite{khalil2002nonlinear}), but it employs a $\beta^2(\cdot)$ gain that may result in overly aggressive control. To address these issues, we propose a continuous LR controller that replaces the discontinuous signum function with a high-slope saturation function. Moreover, inspired by sliding mode control literature, we add an equivalent control term to drive $w$ towards zero.
\end{rem}
We propose the following correction term
\begin{equation}
\label{eq: correction_term}
    u_r=-{\hat\beta(\mathbf{x},v_x,\dot v_x, \dot{\theta}_R)}\text{sat}\left(\tfrac{w}{\epsilon}\right)+u_{eq},
\end{equation}
where the equivalent control term is given by
\begin{equation}
    \label{LR_equivalent_control}
    u_{eq}=-(\mathbf{\hat{B}}^T\mathbf{{P\hat{B}}})^{-1}\mathbf{\hat{B}}^T\mathbf{PA}\boldsymbol{\xi},
\end{equation}
$\hat\beta(\cdot)$ is chosen such that
\begin{equation}
\hat\beta\left(\mathbf{x},v_x,\dot v_x, \dot{\theta}_R\right)\geq\tfrac{\rho\left(\mathbf{x},v_x,\dot v_x, \dot{\theta}_R\right)+\kappa_0|u_{eq}|+\rho_0}{1-\kappa_0},
\end{equation}
$\epsilon$ and $\rho_0$ are small positive constants, and the saturation function is defined by
\begin{equation}
    \text{sat}\left(\tfrac{w}{\epsilon}\right)=\begin{cases}
\tfrac{w}{\epsilon}, & \text{if } |w| 
\leq\epsilon \\
\mathrm{sgn}(w), & \text{if } |w| 
>\epsilon
\end{cases}.
\end{equation}

The following Proposition characterizes the robustness of the LR controller in terms of ultimate boundedness.
\begin{prop}
\label{prop: LR_ultimate}
    With the control law in (\ref{eq: LR_full_feedback}) and the correction term in (\ref{eq: correction_term}), the state $\boldsymbol{\xi}(t)$ of the external dynamics in (\ref{eq: external_LR}) is ultimately bounded by a class $\mathcal{K}$ function of $\epsilon$ for any bounded $\boldsymbol{\xi}(0)$.
\end{prop}
\begin{proof}
From (\ref{eq: external_LR}) and (\ref{eq: LR_w}), we have $\dot w=2\mathbf{\hat B}^T\mathbf{P}\mathbf{\hat A}\boldsymbol{\xi}+2\mathbf{\hat B}^T\mathbf{P}\mathbf{\hat{B}}\left(u_r+d_{lr}(\cdot)\right)$. Consider the Lyapunov function candidate $V_w=\tfrac{1}{2}w^2$, its derivative satisfies
\begin{align}
    \dot V_w=&w\left(2\mathbf{\hat B}^T\mathbf{P}\mathbf{\hat A}\boldsymbol{\xi}+2\mathbf{\hat B}^T\mathbf{P}\mathbf{\hat{B}}\left(u_r+d_{lr}(\cdot)\right)\right) \notag \\
    =& 2w\mathbf{\hat B}^T\mathbf{P}\mathbf{\hat{B}}\Big(-\hat\beta(\cdot)\text{sat}\left(\tfrac{w}{\epsilon}\right)+d_{lr}(\cdot)\Big).
\end{align}
For $|w|>\epsilon$, noting that $\mathbf{\hat B}^T\mathbf{P}\mathbf{\hat{B}}$ is a positive scalar, we have
\begin{align}
    \dot V_w&\leq2\mathbf{\hat B}^T\mathbf{P}\mathbf{\hat{B}}\Big(-|w|\hat\beta+|w|\left(\rho+\kappa_0\left|-\hat\beta\text{sgn}(w)+u_{eq}\right|\right) \Big)
    \notag \\
    &\leq2\mathbf{\hat B}^T\mathbf{P}\mathbf{\hat{B}}\Big(-|w|\hat\beta+|w|\left(\rho+\kappa_0\hat\beta+\kappa_0|u_{eq}|\right) \Big)
    \notag \\
    &= 2\mathbf{\hat B}^T\mathbf{P}\mathbf{\hat{B}}|w|\left(-(1-\kappa_0)\hat\beta+\rho+\kappa_0|u_{eq}|\right) \notag \\
    &\leq-2\mathbf{\hat B}^T\mathbf{P}\mathbf{\hat{B}}\rho_0|w|,
\end{align}
which indicates, by Theorem 4.18 in \cite{khalil2002nonlinear}, that $w$ enters the set $\{|w|\leq\epsilon\}$ in finite time and remains inside thereafter.

Let $p_{ij}$ denote the $(i,j)$-entry of $\mathbf{P}$, from ($\ref{eq: LR_w}$), we have 
\begin{equation}
\label{eq: e_dote_w}
w=2p_{12}e_{lat}+2p_{22}\dot e_{lat}.
\end{equation}
Since $\mathbf{P}$ is positive definite, $p_{22}>0$ by Sylvester’s criterion. Moreover, the $(1,1)$-entry of both sides in (\ref{eq: LR_Lyapunov_equation}) yields the scalar equation $-2k_1 p_{12}=-1$. Since $\mathbf{\hat A}$ is Hurwitz, $k_1>0$. Hence, $p_{12}=\tfrac{1}{2k_1}>0$. Consider the Lyapunov function candidate $V_e=\tfrac{1}{2}e^2_{lat}$, we have
\begin{align}
\label{eq: LR_e_lyap}
    \dot V_e&=e_{lat}(-\tfrac{p_{12}}{p_{22}}e_{lat}+\tfrac{w}{2p_{22}}) \notag \\
    &\leq -\tfrac{p_{12}}{p_{22}}e^2_{lat}+\tfrac{|w||e_{lat}|}{2p_{22}}  \notag \\
    & \leq -\tfrac{\phi p_{12}}{p_{22}}e^2_{lat}, ~~~~\forall |e_{lat}|\geq \tfrac{|w|}{2(1-\phi)p_{12}},
\end{align}
for any constant $\phi\in(0,1)$.

We first show the boundedness of $e_{lat}$ and $\dot e_{lat}$. Since \(w\) enters the set \(\{|w| \leq \epsilon\}\) in finite time and remains there thereafter, and \(w(t)\) is continuous, there exists a constant \(c>0\) such that \(|w(t)|\le c\) for all \(t\ge 0\). By (\ref{eq: LR_e_lyap}) and \(|w(t)|\le c\), we have
$\dot{V}_e \leq -\tfrac{\phi p_{12}}{p_{22}} e_{lat}^2,~ \forall|e_{lat}| > \tfrac{c}{2(1 - \phi)p_{12}}$. Hence, whenever
$|e_{lat}| > \tfrac{c}{2(1 - \phi)p_{12}}$,
the Lyapunov function \(V_e\) is nonincreasing. 
Therefore, the set 
$\{ V_e \leq \tfrac{1}{2} \left( \tfrac{c}{2(1 - \phi)p_{12}} \right)^2 \}$ 
is forward invariant. Since $V_e = \tfrac{1}{2} e_{lat}^2$, this is equivalent to the set 
$\{ |e_{lat}| \leq \tfrac{c}{2(1 - \phi)p_{12}} \}$,
which is thus also forward invariant. Choosing \(c\) sufficiently large so that
$|e_{lat}(0)| \le \tfrac{c}{2(1-\phi)p_{12}}$,
it follows from forward invariance that \(e_{lat}(t)\) remains bounded for all \(t\ge 0\). Moreover, from (\ref{eq: e_dote_w}), $\dot e_{lat}=\frac{w-2p_{12}e_{lat}}{2p_{22}}$ .
Since \(w(t)\) and \(e_{lat}(t)\) are bounded and \(p_{22}>0\), it follows that \(\dot e_{lat}(t)\) is also bounded for all \(t\ge 0\).

Next, we show the ultimate boundedness of $\boldsymbol{\xi}$. After $w$ enters the set $\{|w|\leq\epsilon\}$, \eqref{eq: LR_e_lyap} yields 
\begin{equation}
\dot V_e\leq -\tfrac{\phi p_{12}}{p_{22}}e^2_{lat}, ~\forall |e_{lat}|\geq \tfrac{\epsilon}{2(1-\phi)p_{12}}.
\end{equation}
Since \(w\) enters \(\{|w|\le \epsilon\}\) in finite time and remains there, the above inequality implies that \(e_{lat}\) enters
\(\left\{|e_{lat}|\le \tfrac{\epsilon}{2(1-\phi)p_{12}}\right\}
\)
in finite time and remains there. Hence, all trajectories reach
\(
\Omega=\left\{|e_{lat}|\le\tfrac{\epsilon}{2(1-\phi)p_{12}},\; |w|\le\epsilon\right\}
\)
in finite time and remain inside thereafter.

Inside $\Omega$, by (\ref{eq: e_dote_w}), $
|\dot e_{lat}|=\left|-\tfrac{p_{12}}{p_{22}}e_{lat}+\tfrac{w}{2p_{22}}\right|\leq\tfrac{(2-\phi)\epsilon}{2(1-\phi)p_{22}}.$ Consequently, all trajectories enter the ultimate bound, 
\begin{equation}
    \|\boldsymbol{\xi}\|_2=\sqrt{e^2_{lat}+\dot e^2_{lat}}\leq \tfrac{\sqrt{p_{22}^2+p^2_{12}(2-\phi)^2}}{2(1-\phi)p_{12}p_{22}}\epsilon,
\end{equation}
in finite time. Since the ultimate bound is linear in \(\epsilon\) with positive coefficient, it is a class \(\mathcal K\) function of \(\epsilon\).
\end{proof}
\begin{rem}
\label{rem: LR_epsilon}
    Proposition \ref{prop: LR_ultimate} offers a clear tuning guideline: choosing smaller $\epsilon$ values enhances robustness.
\end{rem}

\subsection{Incremental Nonlinear Dynamic Inversion}
INDI derives its name from nonlinear dynamic inversion, an alternative term for feedback linearization commonly used in the aerospace community. The INDI method replaces part of the model-based dynamics with measured signals, thereby reducing reliance on precise system modeling. Its effectiveness and robustness have been evaluated on unmanned aerial vehicles via both simulations and experiments \cite{wang2019stability, tal2020accurate, sun2020incremental,sun2022comparative}.

The INDI controller updates the control signal incrementally. From (\ref{eq: output_derivative_1})-(\ref{eq: output_derivative_2}), if we can achieve $\ddot e_{lat}(t)=-\mathbf{K}\boldsymbol{\xi}(t)$, the resulting external dynamics 
\begin{equation}
    \boldsymbol{\dot\xi}(t)={\small\begin{bmatrix}
0 & 1 \\
-k_1 & -k_2
\end{bmatrix}}\boldsymbol{\xi}(t)
\end{equation}
can be easily stabilized. Thus, we treat $-\mathbf{K}\boldsymbol{\xi}(t)$ as the desired value of $\ddot e_{lat}(t)$. Let $\mathbf{y}=[\dot e_{lat}, \tilde\theta,\dot{\tilde\theta},\dot\theta_R,v_x,\dot v_x]^T$ and $\tau$ be a small positive constant. For all $t \geq 0$, we perform the Taylor series expansion of (\ref{eq: e_ddot_govern}) around the point $\left(\mathbf{y}(t-\tau),\delta(t-\tau)\right)$:
\begin{align}
\label{ddot_e_discrete}
    \ddot e_{lat}(t) = &\ddot e_{lat}(t-\tau) +g_1\Delta\delta_f(t) \notag \\
    &+\underbrace{\left.\tfrac{\partial(f_1+d_1\dot\theta_R)}{\partial\mathbf{y}}\right|_{\mathbf{y}=\mathbf{y}(t-\tau)}\Delta\mathbf{y}(t)+R_1(t)}_{\sigma(t)}, 
\end{align}
where $\Delta\delta_f(t)=\delta_f(t)-\delta_f(t-\tau)$, $\Delta\mathbf{y}(t)=\mathbf{y}(t)-\mathbf{y}(t-\tau)$,
\begin{align}
\label{remainder}
    R_1(t) = \tfrac{1}{2}\Delta\mathbf{y}^T(t)\left.\tfrac{\partial^2(f_1+d_1\dot\theta_R)}{\partial\mathbf{y}^2}\right|_{\mathbf{y}=\mathbf{y^*}}\Delta\mathbf{y}(t)
\end{align} 
denotes the expansion remainder wherein $\mathbf{y^*}$ belongs to the line segment joining $\mathbf{y}(t-\tau)$ and $\mathbf{y}(t)$. Note that $\ddot e_{lat}$ is linear in $\delta_f$ so $R_1$ does not contain terms dependent on $\delta_f$. 

In the INDI method, the expression of $\ddot e_{lat}(t)$ are partly replaced with the measurement $\ddot e_{lat}(t-\tau)$, and the control input is updated incrementally using the measurement $\delta_f(t-\tau)$. Recall that the desired value of $\ddot e_{lat}(t)$ is $-\mathbf{K}\boldsymbol{\xi}(t)$, we set the incremental control term as
\begin{equation}
\label{eq: increm_law}
    \Delta\delta_f(t) =\tfrac{1}{\hat g_1}\left(-\mathbf{K}\boldsymbol{\xi}(t)-\ddot e_{lat}(t-\tau)\right),
\end{equation}
and the full INDI control law
\begin{equation}
\label{eq: indi_full_control}
    \delta_f(t) = \delta_f(t-\tau)+\Delta\delta_f(t).
\end{equation}
Substituting (\ref{eq: increm_law}) into (\ref{ddot_e_discrete}) yields
\begin{equation}
\label{eq: 62}
    \ddot e_{lat}(t)=-\mathbf{K}\boldsymbol{\xi}(t)+\underbrace{\left(1-\tfrac{g_1}{\hat g_1}\right)\left(\mathbf{K}\boldsymbol{\xi}(t)+\ddot e_{lat}(t-\tau)\right)+\sigma(t)}_{d_{in}(t)}.
\end{equation}
From \eqref{eq: 62}, we have $\mathbf{K}\boldsymbol{\xi}(t)=d_{in}(t)-\ddot e_{lat}(t)$, substituting it for $\mathbf{K}\boldsymbol{\xi}(t)$ in the expression of $d_{in}(t)$ in \eqref{eq: 62} yields
\begin{align}
\label{indi_d_in}
    d_{in}(t)=\left(1-\tfrac{g_1}{\hat g_1}\right)&\bigl(d_{in}(t)-\ddot e_{lat}(t)+\ddot e_{lat}(t-\tau)\bigr)+\sigma(t), \notag\\
\implies\tfrac{g_1}{\hat g_1}d_{in}(t)
&=\left(1-\tfrac{\hat g_1}{g_1}\right)\Delta \ddot e_{lat}(t)+\tfrac{\hat g_1}{g_1}\sigma(t),
\end{align}
where, $\Delta \ddot e_{lat}(t)=\ddot e_{lat}(t)-\ddot e_{lat}(t-\tau)$. Finally, substituting $\ddot e_{lat}(t)=-\mathbf{K}\boldsymbol{\xi}(t)+d_{in}(t)$ into (\ref{eq: output_derivative_1})-(\ref{eq: output_derivative_2}) gives
\begin{equation}
\label{eq: external_dynamics_INDI}
    \boldsymbol{\dot\xi}(t)=\underbrace{{\small\begin{bmatrix}
0 & 1 \\
-k_1 & -k_2
\end{bmatrix}}}_{\mathbf{A_{in}}}\mathbf{\xi}(t)+\underbrace{{\small\begin{bmatrix}
    0\\1
\end{bmatrix}}}_{\mathbf{B_{in}}}d_{in}(t).
\end{equation}
\begin{rem}
From \eqref{eq: increm_law}--\eqref{eq: indi_full_control}, it follows that INDI does not require knowledge of parameter bounds and relies only on the nominal value of \(g_1\). This makes the controller semi-model-free and alleviates parameter identification efforts. Compared to LR, INDI uses additional onboard measurements \cite{wang2019stability}. In particular, the steering angle \(\delta_f\), typically accessible via the vehicle CAN bus, is needed. Moreover, computing \(\ddot e_{lat}\) requires \(\dot v_y\) (see \eqref{eq: expression_elat_ddot}), which can be obtained from the IMU. Such measurements are commonly available on AV platforms.
\end{rem}

\begin{rem}
    In existing literature, the analysis of the INDI control law primarily relies on either simplification through Jacobian linearization \cite{tal2020accurate}, or assuming the perturbation term $d_{in}$ is bounded over the entire state space \cite{sun2020incremental,wang2019stability}. In contrast, we do not invoke such simplifications or assumptions. Instead, we explicitly account for the structure of $d_{in}$ and establish a semi-global ultimate boundedness result.
\end{rem}
To analyze the ultimate boundedness of the external dynamics, we need to first establish the uniform boundedness of the perturbation term $d_{in}$ in the following lemma. Note that we define a closed (Euclidean) ball as $
\overline{\mathbb{B}}_r := \left\{ \mathbf{z} \in \mathbb{R}^n \;\middle|\; \|\mathbf{z}\|_2 \leq r \right\}$.
\begin{lem}
\label{lem: bounded_din}
     For the external dynamics in (\ref{eq: external_dynamics_INDI}), if: (i) $\boldsymbol{\xi}(t) \in \overline{\mathbb{B}}_r$ for some finite $r > 0$, for all $t \geq -\tau$; (ii) $\ddot e_{lat}(t)$ and $\delta_f(t)$ are continuous on $t\in[-\tau,0]$, then for any $\overline{d}_{\mathrm{in}} > 0$, there exists $\tau^* > 0$ such that for all $t \geq 0$, if $0 < \tau < \tau^*$, we have $|d_{in}(t)| < \overline{d}_{in}$.

\end{lem}
\begin{proof}
Referring back to Assumption \ref{assump: longitudinal}, the values for $v_x$, $\dot v_x$, and $\dot\theta_R$ remain in compact intervals for all $t\in\mathbb{R}$, and $v_x$ is bounded away from $0$. Furthermore, $\boldsymbol{\xi}$ is assumed to lie within a closed ball for all $t\geq-\tau$. 
 
 We first show the uniform boundedness of the internal state $\boldsymbol{\eta}(t)$. By Proposition \ref{prop: internal_ISS}, the internal dynamics is ISS with respect to the inputs \((\boldsymbol{\xi},\dot\theta_R)\). Hence, for all \(t\ge -\tau\),Proposition \ref{prop: internal_ISS} indicates that \cite{khalil2002nonlinear}, for all $t\geq-\tau$,
\begin{equation}
\|\boldsymbol{\eta}(t)\|_2 \leq \beta_\eta\left(\|\boldsymbol{\eta}(-\tau)\|_2, t + \tau \right)
+ \gamma_\eta\Big(\sup_{-\tau \leq t_1 \leq t} \left\| 
\begin{smallmatrix} 
\boldsymbol{\xi}(t_1) \\ 
\dot{\theta}_R(t_1) 
\end{smallmatrix} 
\right\|_2 \Big),
\end{equation}
where $\beta_\eta \in \mathcal{KL}$ and $\gamma_\eta \in \mathcal{K}$. Since $\boldsymbol{\xi}(t) \in \overline{\mathbb{B}}_r$ and $\dot{\theta}_R(t) \in [\dot{\theta}_{R,\min}, \dot{\theta}_{R,\max}]$ for all $t \geq -\tau$, the supremum term in the above equation is uniformly bounded in \(t\). Furthermore,
\begin{equation}
\beta_\eta(\|\boldsymbol{\eta}(-\tau)\|_2, t + \tau) \leq \beta_\eta(\|\boldsymbol{\eta}(-\tau)\|_2, 0),
\end{equation}
hence, $\boldsymbol{\eta}(t)$ belongs to some closed ball $\overline{\mathbb{B}}_{r_1} \forall t\geq-\tau$.

Since $\boldsymbol{\xi}(t)$ and $\boldsymbol{\eta}(t)$ belong to $\overline{\mathbb{B}}_{r}$ and $\overline{\mathbb{B}}_{r_1}$, respectively, we have $\left(\boldsymbol{\xi}(t),\boldsymbol{\eta}(t)\right)\in$ $\overline{\mathbb{B}}_{r_2} \forall t\geq-\tau$, with $r_2=\sqrt{r^2+r_1^2}$. Furthermore, by the diffeomorphism in \eqref{eq: coord_trans}, $\mathbf{x}(t)=M^{-1}(\boldsymbol{\xi}(t),\boldsymbol{\eta}(t))$, so $\mathbf{x}(t)\in\overline{\mathbb{B}}_{r_3}$ $\forall t\geq-\tau$ where $r_3=\|M^{-1}\|_2r_2$. Hence, 
$\left(\dot e_{lat}(t),\tilde\theta(t),\dot{\tilde\theta}(t)\right)\in\overline{\mathbb{B}}_{r_3} \forall t\geq-\tau.$
 
  From the above arguments, $\mathbf{y}(t)$ and $\mathbf{y}(t-\tau)$ belong to a convex compact set, $\mathcal{Y}=\overline{\mathbb{B}}_{r_3} \times[\dot{\theta}_{R,\min}, \dot{\theta}_{R,\max}] \times[v_{x,min},v_{x,max}] \times[\dot v_{x,min},\dot v_{x,max}]$, for all $t\geq0$. Since $\mathbf{y}^*(t)$ lies on the line segment connecting $\mathbf{y}(t-\tau)$ and $\mathbf{y}(t)$, and $\mathcal{Y}$ is convex, it follows that $\mathbf{y}^*(t) \in \mathcal{Y}$ for all $t \geq 0$. Therefore, for the following continuous functions 
 \begin{align}
 &h_1(\mathbf{y}(t-\tau)):=\left.\tfrac{\partial(f_1+d_1\dot\theta_R)}{\partial\mathbf{y}}\right|_{\mathbf{y}=\mathbf{y}(t-\tau)}, \notag \\
 &h_2(\mathbf{y^*}(t)):=\left.\tfrac{\partial^2(f_1+d_1\dot\theta_R)}{\partial\mathbf{y}^2}\right|_{\mathbf{y}=\mathbf{y^*}},
 \end{align}
the arguments of $h_1(\cdot)$ and $h_2(\cdot)$ lie in the compact set $\mathcal{Y}$ for all $t \geq 0$. It follows that $h_1(\cdot)$ and $h_2(\cdot)$ are uniformly bounded for all $t \geq 0$.

Although $\sigma$ and $d_{\mathrm{in}}$ were originally defined in \eqref{ddot_e_discrete} and \eqref{indi_d_in}, respectively, as functions of time $t$ given a fixed $\tau$, we now reinterpret them as functions of both $t$ and $\tau$. Specifically,
\begin{equation}
\label{eq: sigma_redefined}
    \sigma(t,\tau):= h_1(t,\tau)\Delta\mathbf{y}(t,\tau)+\tfrac{1}{2}\Delta\mathbf{y}^T(t,\tau)h_2(t,\tau)\Delta\mathbf{y}(t,\tau), 
\end{equation}
\begin{equation}
\label{eq: d_in_redefined}
    d_{in}(t,\tau):=\left(1-\tfrac{\hat g}{g}\right)\Delta\ddot e_{lat}(t,\tau)+\tfrac{\hat g}{g}\sigma(t,\tau).
\end{equation}
Since all components of $\mathbf{y}(t)$ are continuous in time, we have $\lim_{\tau\to0^+}\Delta\mathbf{y}(t,\tau)=\mathbf{0}$ for every $t\geq0$. 

Since $\ddot e_{lat}$ and $\delta_f$ are assumed continuous on $t\in[-\tau,0]$, by \eqref{eq: increm_law}-\eqref{eq: indi_full_control}, $\delta_f$ is continuous on $t\in[0,\tau]$. Hence, $\ddot e_{lat}$ is continuous on $t\in[0,\tau]$ by \eqref{eq: e_ddot_govern}. Iterating this argument shows that $\ddot e_{lat}$ is continuous for all $t\geq-\tau$. Therefore, $\lim_{\tau\to0^+}\Delta\ddot e_{lat}(t,\tau)=\ddot e_{lat}(t)-\ddot e_{lat}(t)=0$ $\forall t\geq0$. 

Since $h_1(\cdot)$ and $h_2(\cdot)$ are uniformly bounded for all $t\geq 0$, we have $\lim_{\tau\to0^+}d_{in}(t,\tau)=0$ for all $t\geq0$ by \eqref{eq: sigma_redefined}--\eqref{eq: d_in_redefined}. Accordingly, for every $\overline{d}_{in} > 0$, there exists a sufficiently small $\tau^* > 0$ such that, for all $t\geq0$, if $0<\tau < \tau^*$, then $|d_{in}(t,\tau)| < \overline{d}_{in}$ \cite{wang2019stability}.
\end{proof}

Utilizing Lemma \ref{lem: bounded_din}, we have the following proposition that establishes the robustness of the INDI controller. Note that since $\mathbf{K}$ is chosen such that $\mathbf{A_{in}}$ in \eqref{eq: external_dynamics_INDI} is Hurwitz, there exists a positive definite matrix $\mathbf{P}$ satisfying the Lyapunov equation $\mathbf{P}\mathbf{A_{in}}+\mathbf{A_{in}}^T\mathbf{P}=-\mathbf{I}$.

\begin{prop}
\label{prop: INDI}
For the external dynamics in (\ref{eq: external_dynamics_INDI}), if $\ddot e_{lat}(t)$ and $\delta_f(t)$ are continuous on $t\in[-\tau,0]$, then for any $r > 0$,  there exists a constant $\tau^* > 0$ such that for all $0 < \tau < \tau^*$, the state $\boldsymbol{\xi}(t)$ 
is ultimately bounded by an arbitrarily small value for every initial state satisfying $\|\boldsymbol{\xi}(0)\|_2 \leq \sqrt{\tfrac{\lambda_{\min}(\mathbf{P})}{\lambda_{\max}(\mathbf{P})}} r$.

\end{prop}
\begin{proof}
Taking the Lyapunov function candidate $V=\boldsymbol{\xi}^T\mathbf{P}\boldsymbol{\xi}$, 
\begin{equation}
\implies\lambda_{min}(\mathbf{P})\|\boldsymbol{\xi}\|_2^2\leq V\leq \lambda_{max}(\mathbf{P})\|\boldsymbol{\xi}\|_2^2.
\end{equation}
Let $\phi$ be a constant and $\phi\in(0,1)$. Choose any \(\overline d_{in}>0\) satisfying
\begin{equation}
\label{eq: d_in_bounded}
\overline{d}_{in}<\tfrac{(1-\phi)\sqrt{\lambda_{min}(\mathbf{P})}}{2\|\mathbf{P}\mathbf{B}_{\mathbf{in}}\|_2\sqrt{\lambda_{max}(\mathbf{P})}}\,r.
\end{equation}
Then, by Lemma \ref{lem: bounded_din}, there exists \(\tau^*>0\) such that if \(0<\tau<\tau^*\), we have \(|d_{in}(t)|<\overline d_{in}\) for all \(t\ge0\), provided \(\boldsymbol{\xi}(t)\in\overline{\mathbb B}_r\) for all \(t\ge0\). Then, for all \(\boldsymbol{\xi}(t)\in\overline{\mathbb{B}}_r\), we have
\begin{align}
\label{eq: Lyap_INDI}
    \dot V=&(\mathbf{A_{in}}\boldsymbol{\xi}+\mathbf{B_{in}}d_{in})^T\mathbf{P}\boldsymbol{\xi} + \boldsymbol{\xi}^T\mathbf{P}(\mathbf{A_{in}}\boldsymbol{\xi}+\mathbf{B_{in}}d_{in}) \notag \\
    =&-\|\boldsymbol{\xi}\|_2^2+2d_{in}\mathbf{B_{in}}^T\mathbf{P}\boldsymbol{\xi} \notag\\
    <&-\|\boldsymbol{\xi}\|_2^2+2\overline d_{in}\cdot\|\mathbf{P} \mathbf{B}_{\mathbf{in}}\|_2\cdot \|\boldsymbol{\xi}\|_2 \notag\\
    \leq&-\phi\|\boldsymbol{\xi}\|_2^2,~~~~\forall \|\boldsymbol{\xi}\|_2\geq \tfrac{2\|\mathbf{P} \mathbf{B}_{\mathbf{in}}\|_2 \cdot  \overline d_{in}}{1-\phi} :=\mu_1\overline d_{in}.
\end{align}
Furthermore, by \eqref{eq: d_in_bounded}, we have $ \mu_1\overline d_{in}<\sqrt{\tfrac{\lambda_{min}(\mathbf{P})}{\lambda_{max}(\mathbf{P})}}r$.

Since the last inequality in (\ref{eq: Lyap_INDI}) holds for all $\boldsymbol{\xi}(t)\in\overline{\mathbb{B}}_r$ and $ \mu_1\overline d_{in}<\sqrt{\tfrac{\lambda_{min}(\mathbf{P})}{\lambda_{max}(\mathbf{P})}}r$ by \eqref{eq: d_in_bounded}, according to Theorem 4.18 in \cite{khalil2002nonlinear},  for every initial state $\boldsymbol{\xi}(0)$ satisfying $\|\boldsymbol{\xi}(0)\|_2\leq\sqrt{\tfrac{\lambda_{min}(\mathbf{P})}{\lambda_{max}(\mathbf{P})}}r$, the state $\boldsymbol{\xi}(t)$ has an ultimate bound $\sqrt{\tfrac{\lambda_{max}(\mathbf{P})}{\lambda_{min}(\mathbf{P})}}\mu_1\overline{d}_{in}$, which is a class $\mathcal{K}$ function of $\overline{d}_{in}$. 

By Lemma~\ref{lem: bounded_din}, for any sufficiently small \(\overline d_{in}>0\) satisfying \eqref{eq: d_in_bounded}, there exists a corresponding \(\tau^*>0\) such that, if \(0<\tau<\tau^*\), then \(|d_{in}(t)|<\overline d_{in}\) for all \(t\ge0\). Therefore, the ultimate bound of \(\boldsymbol{\xi}(t)\) can be made arbitrarily small by choosing \(\tau\) sufficiently small. This proves the proposition
\end{proof}
\begin{rem}
Proposition \ref{prop: INDI} does not impose an a priori restriction on the magnitude of \(r\). Hence, the admissible set of initial states, defined by \(\|\boldsymbol{\xi}(0)\|_2 \leq \sqrt{\tfrac{\lambda_{\min}(\mathbf{P})}{\lambda_{\max}(\mathbf{P})}}\, r\), can be made arbitrarily large. However, the corresponding constant \(\tau^*\) may depend on the choice of \(r\). Therefore, Proposition \ref{prop: INDI} establishes semi-global ultimate boundedness.
\end{rem}
\begin{rem}
\label{rem: INDI_tau}
Although Proposition \ref{prop: INDI} does not explicitly characterize how the ultimate bound on \(\boldsymbol{\xi}(t)\) varies with \(\tau\) (as in Remark \ref{rem: LR_epsilon}), \eqref{eq: sigma_redefined}--\eqref{eq: d_in_redefined} indicate that \(d_{in}(t,\tau)\) depends on the difference terms \(\Delta \mathbf{y}(t,\tau)\) and \(\Delta \ddot e_{lat}(t,\tau)\), both of which decrease as \(\tau\to 0\). This suggests that choosing a smaller \(\tau\) tends to produce a smaller bound \(\overline d_{in}\), and hence a tighter ultimate bound on \(\boldsymbol{\xi}(t)\) (which is a class \(\mathcal{K}\) function of \(\overline d_{in}\)).
\end{rem}

\section{Simulation Results}
\label{sec5}
To systematically assess the proposed controllers across diverse operating conditions, simulations are performed in Simulink using vehicle parameters identified from a Lincoln MKZ \cite{liu2020lateral}. Unlike the real-vehicle experiment presented later, these simulations allow extensive testing across different speeds, parameter uncertainties, and controller settings in a safe, controlled environment. A dual-track vehicle module with aerodynamics from the Simulink Automated Driving Toolbox is used to obtain a realistic test platform. Fig. \ref{fig: track}(a) shows the test track and reference path: the gray area is the track, the blue line is the reference path, and the blue box is the vehicle’s initial position. Starting near the track center, the vehicle completes one lap with four lane changes. The path’s peak curvature is about $0.012\,\mathrm{m}^{-1}$, as shown in Fig. \ref{fig: track}(b).
\begin{figure}[h]
    \centering
    \subfigure[test track and reference path]{\includegraphics[width=0.215\textwidth]{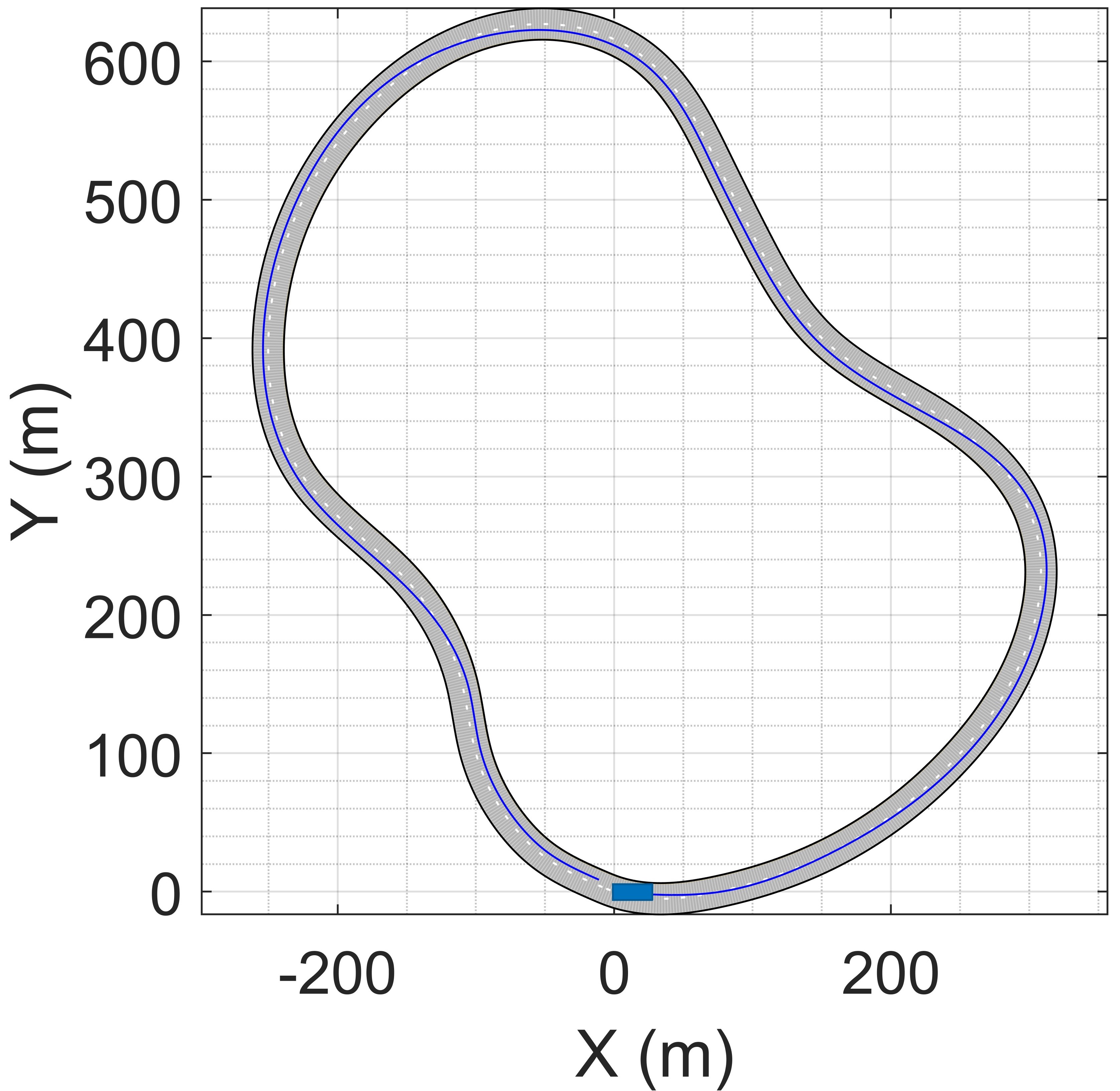}}
    \subfigure[curvature profile]{\includegraphics[width=0.26\textwidth]{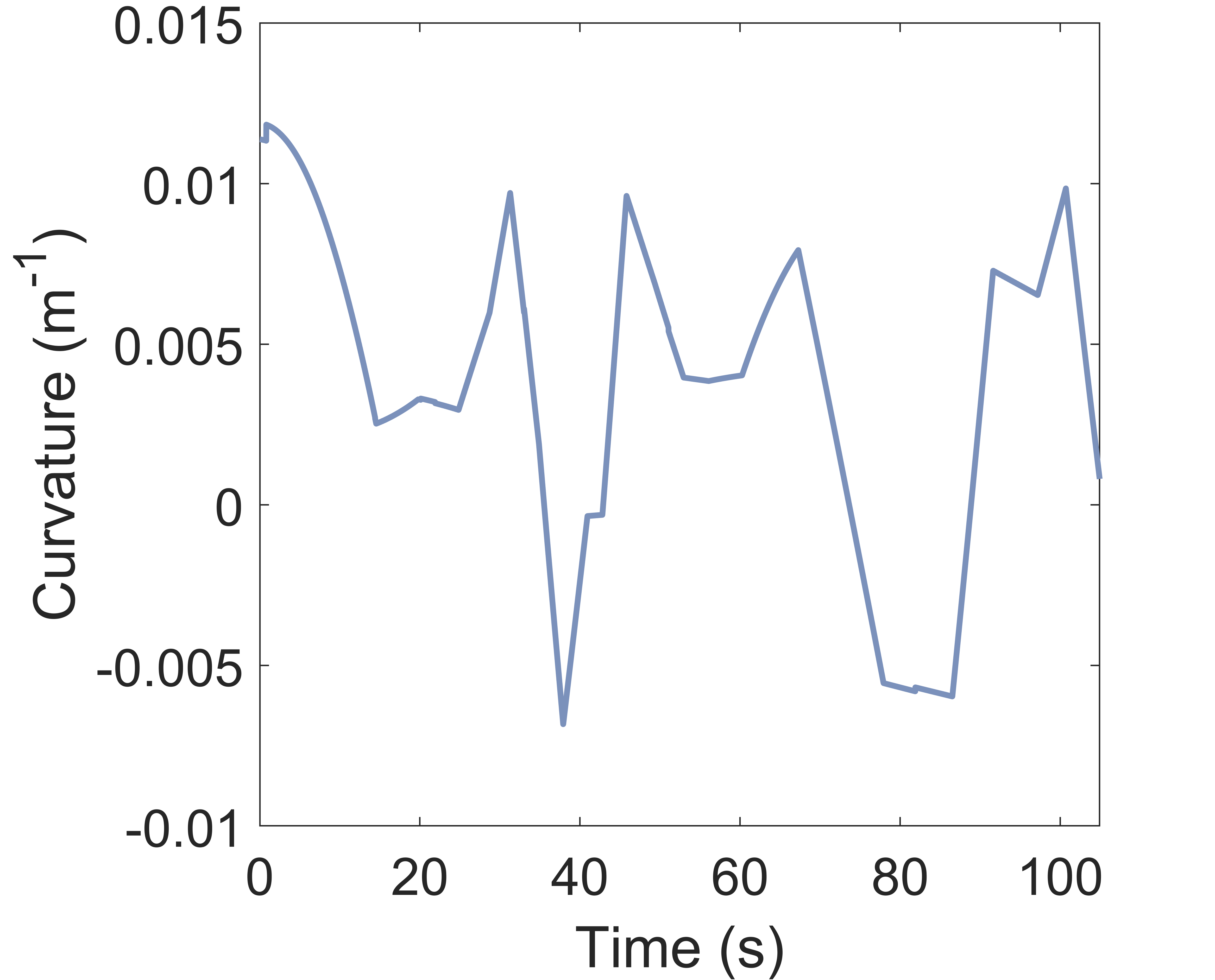}}
    \caption{Test track, reference path, and curvature profile}
    \label{fig: track}
    \end{figure}

The nominal parameter values and their upper and lower limits are listed in Table \ref{table: parameters}. These nominal values are obtained offline from a Lincoln MKZ \cite{liu2020lateral}. Following \cite{peng1990lateral}, cornering stiffness is the dominant source of vehicle-dynamics uncertainty. Therefore, $C_f$ and $C_r$ are allowed to deviate by up to $80\%$ from their nominal values, while the uncertainties in $m$ and $I_z$ remain relatively minor, within $30\%$. As stated in Assumption \ref{assump: parameters}, $a$ and $b$ are considered fixed and known.
\begin{table}[ht]
\centering
\caption{Parameters of the Lincoln MKZ}
\resizebox{1\linewidth}{!}{
\begin{tabular}{cccccc}
\hline
\hline
\textbf{Parameter} & \textbf{Symbol} & \textbf{Unit} & \textbf{Nominal} & \makecell{\textbf{Min}\\\textbf{(relative)}} & \makecell{\textbf{Max}\\\textbf{(relative)}} \\
\hline
Mass & $m$ & kg & 1896 & 1 & 1.3 \\
Moment of Inertia & $I_z$ & kg$\cdot$m$^2$ & 3803 & 1 & 1.3 \\
Front Cornering Stiffness & $C_{f}$ & N/rad & 400000 & 0.2 & 1.8 \\
Rear Cornering Stiffness & $C_{r}$ & N/rad & 381900 & 0.2 & 1.8 \\
C.G. to front axle & $a$ & m & 1.2682  & -- & -- \\
C.G. to rear axle & $b$ &m & 1.5818  & -- & -- \\
\hline
\hline
\label{table: parameters}
\end{tabular}}
\end{table}

The LR controller uses the measurements $e_{lat}$, $\dot e_{lat}$, $\tilde\theta$, $\dot{\tilde\theta}$, $\dot{\theta}_R$, $v_x$, and $\dot v_x$. With the parameter ranges in Table \ref{table: parameters}, the coefficients in \eqref{eq: purterbation bound} can be evaluated as
\begin{equation}
    \label{eq: alpha_bound_value}
 \max(|\alpha_\Delta|)=\max\left(\left|\tfrac{C_f\hat m}{\hat C_f m}-1\right|\right)=\left|\tfrac{(0.2\hat C_f)\hat m}{\hat C_f(1.3\hat m)}-1\right|={0.846},
\end{equation}
\begin{align}
    \label{eq: gamma_bound_value}\max\left(|\gamma_\Delta|\right)&=\max\left(\left|\tfrac{1}{m}\left(C_r-C_f\tfrac{\hat C_r}{\hat C_f}\right)\right|\right) \notag \\
    &=\left|\tfrac{1}{\hat m}\left(0.2\hat C_r -1.8\hat C_f\tfrac{\hat C_r}{\hat C_f} \right)\right|=322.3.
\end{align}
Substituting \eqref{eq: alpha_bound_value} and \eqref{eq: gamma_bound_value} into \eqref{eq: purterbation bound} and comparing the result with \eqref{eq: perturb_bound} yields the expression for $\rho(\cdot)$ and the value of $\kappa_0$.

In contrast, the INDI controller only assumes that the parameters are bounded, without needing their exact limits. It does, however, use additional measurements of $\ddot e_{lat}$ and $\delta_f$.

Unless stated otherwise, the proposed controllers use the tuned parameters in Table~\ref{table: controller_parameters}. These defaults are chosen to provide a good balance between tracking accuracy and control smoothness in the upcoming nominal test.
\begin{table}[ht]
\centering
\scriptsize
\caption{Default Parameters of the Proposed Controllers}
\begin{tabular}{c c c c| c c c}
\hline\hline
\multicolumn{4}{c|}{\textbf{LR Controller}} & \multicolumn{3}{c}{\textbf{INDI Controller}} \\
\hline
$k_1$ & $k_2$ & $\epsilon$ & $\rho_0$ & $k_1$ & $k_2$ & $\tau$ \\
0.74  & 4.81  & 0.35       & 0.4      & 5.66  & 10.09 & 0.03  \\
\hline\hline
\end{tabular}
\label{table: controller_parameters}
\end{table}
\subsection{Nominal test with varying speed and acceleration}
This subsection assesses how incorporating longitudinal motion awareness into the lateral controller affects performance under varying speed and acceleration. The Simulink vehicle uses the nominal parameter values in Table \ref{table: parameters}. The vehicle follows the speed profile in Fig. \ref{fig: vel_prof}: it accelerates to $30\, \text{m/s}$, holds this speed, then decelerates to $10\, \text{m/s}$ and maintains it until the test concludes. This profile is selected to expose the controller to a wide speed range and both acceleration and deceleration phases, thereby directly evaluating the impact of longitudinal-motion awareness on lateral tracking.
    \begin{figure}[h]
    \centering
    \subfigure[speed]{\includegraphics[width=0.23\textwidth]{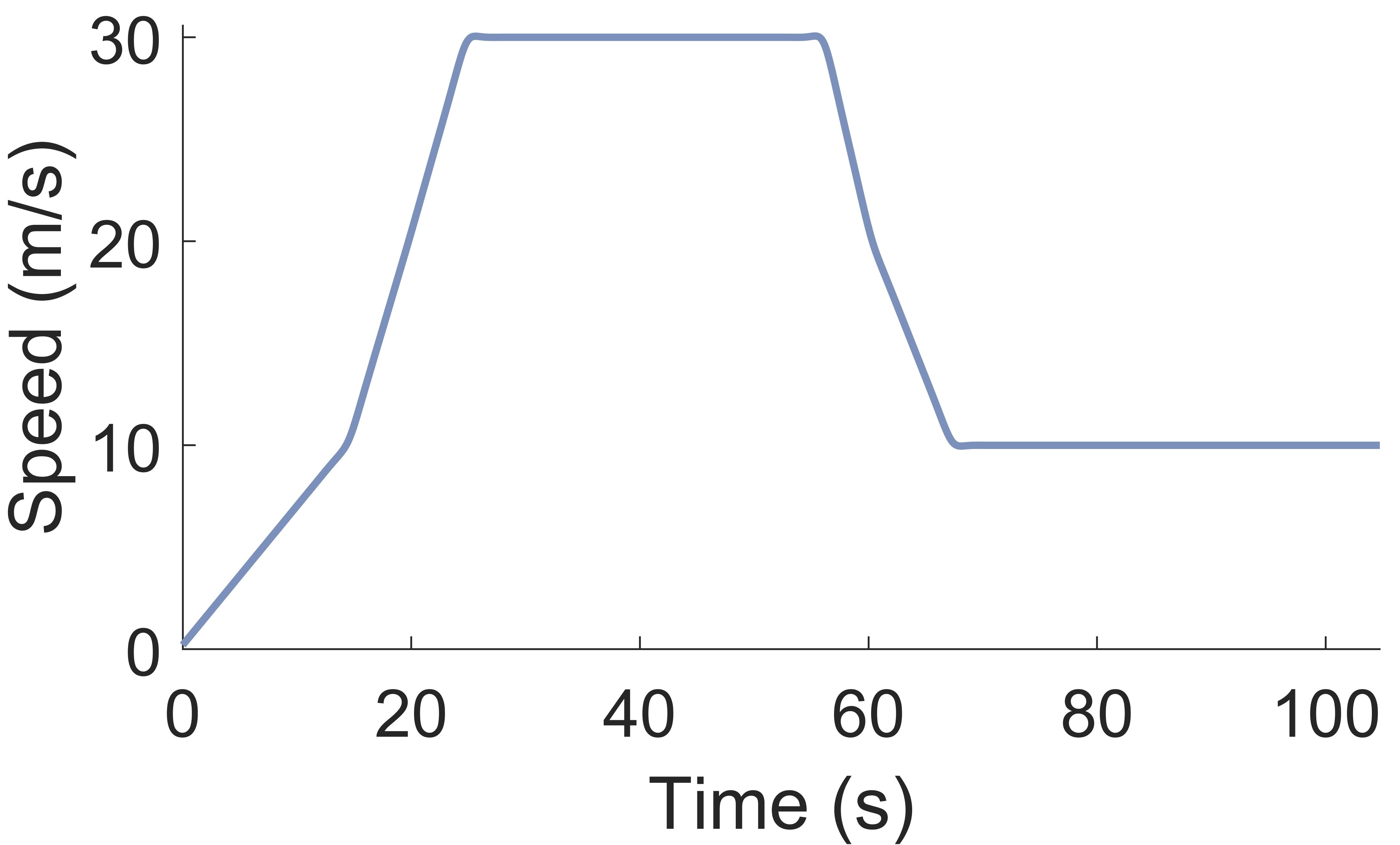}}
    \subfigure[acceleration]{\includegraphics[width=0.23\textwidth]{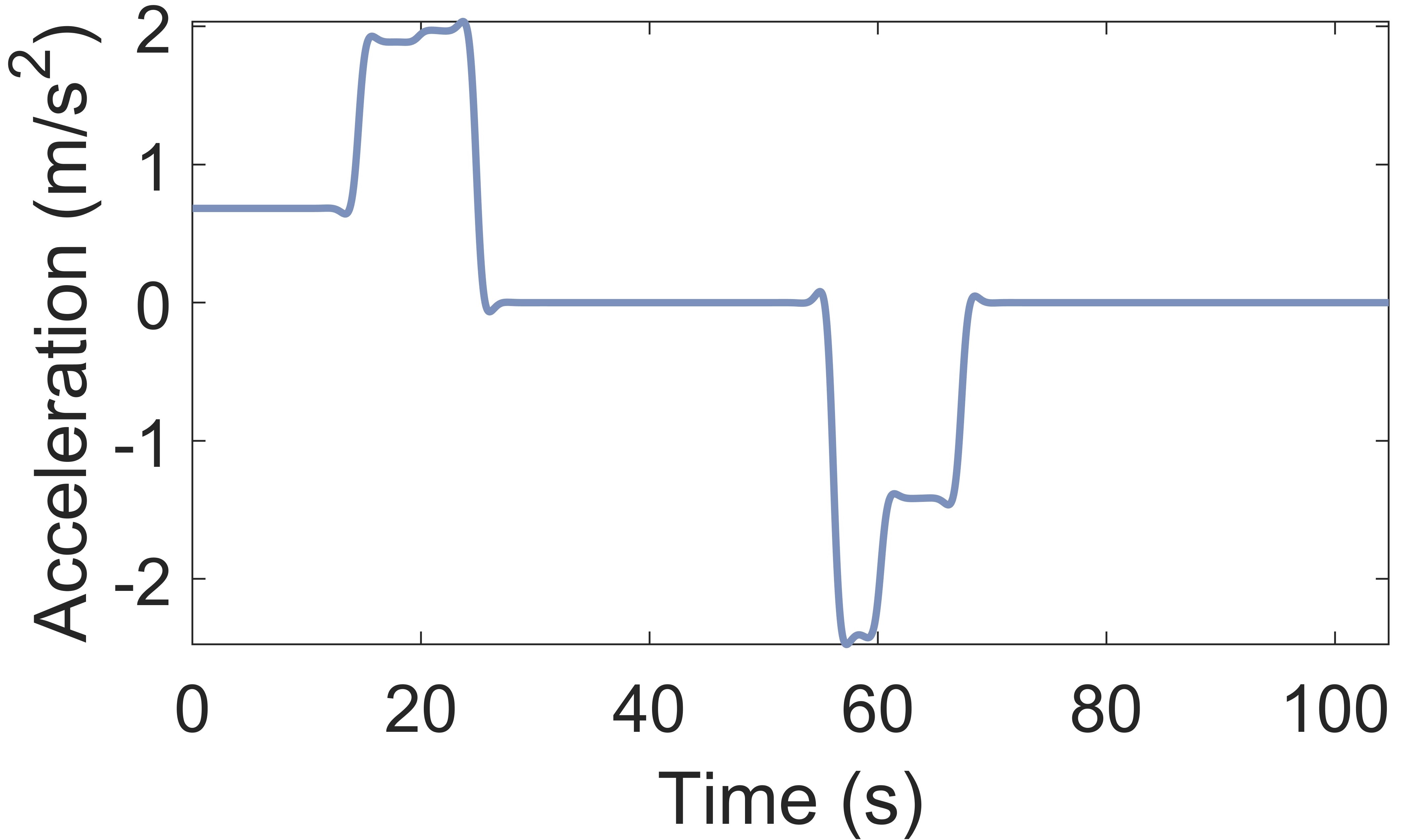}}
    \caption{Speed and acceleration profile of the vehicle}
    \label{fig: vel_prof}
    \end{figure}

The proposed controllers are compared against a widely used feedback–feedforward (FF) lateral controller \cite{peng1992vehicle,rajamani2011vehicle,liu2020lateral}. As a baseline, we adopt the specific design in \cite{liu2020lateral}, which models longitudinal speed as an unknown but bounded parameter and selects feedback gains to ensure stability for all speeds in $(0,30]\,\mathrm{m/s}$. Originally derived for constant speeds, this controller was later shown to remain stabilizing under varying speeds \cite{liu2020robust}. Because this baseline was developed for the same Lincoln MKZ, its gains can be applied here directly without retuning. The aim of this comparison is not to exhaustively benchmark all controller classes, but to demonstrate the benefit of explicitly incorporating longitudinal speed and acceleration in the lateral control law. Thus, this representative baseline, which treats speed variation in a conventional manner, offers a clearer reference than alternatives based on substantially different design paradigms, which could introduce confounding factors. We refer to it as the FF controller.

For the LR, INDI, and FF controllers, the tracking errors $e_{lat},~\dot e_{lat}, ~\tilde\theta$, $\dot{\tilde\theta}$, and the control $\delta_f$ are plotted in Fig. \ref{fig: nominal_test}.
    \begin{figure}[h]
    \centering
    {\includegraphics[width=0.49\textwidth]{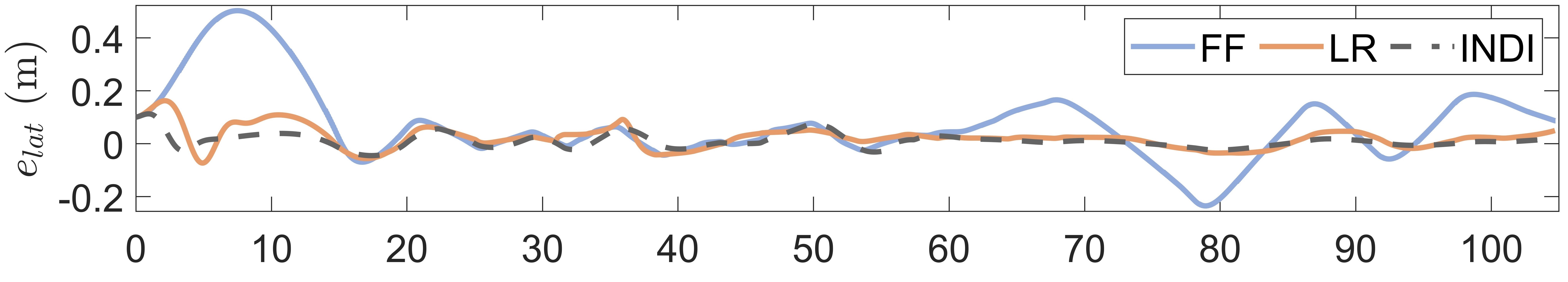}}
    {\includegraphics[width=0.49\textwidth]{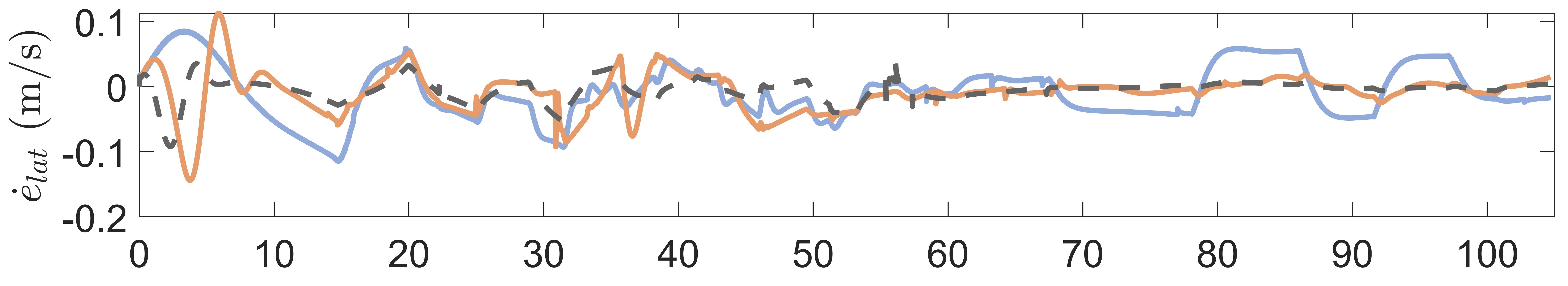}}
    {\includegraphics[width=0.49\textwidth]{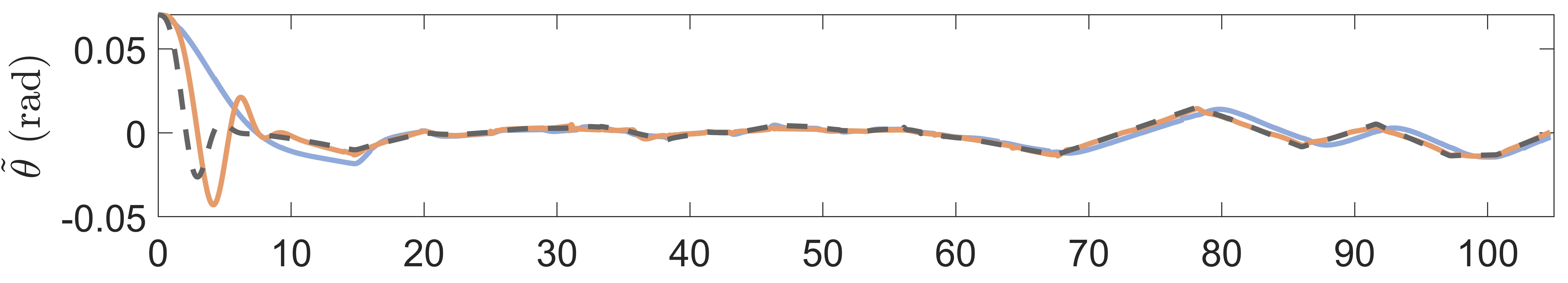}}
    {\includegraphics[width=0.49\textwidth]{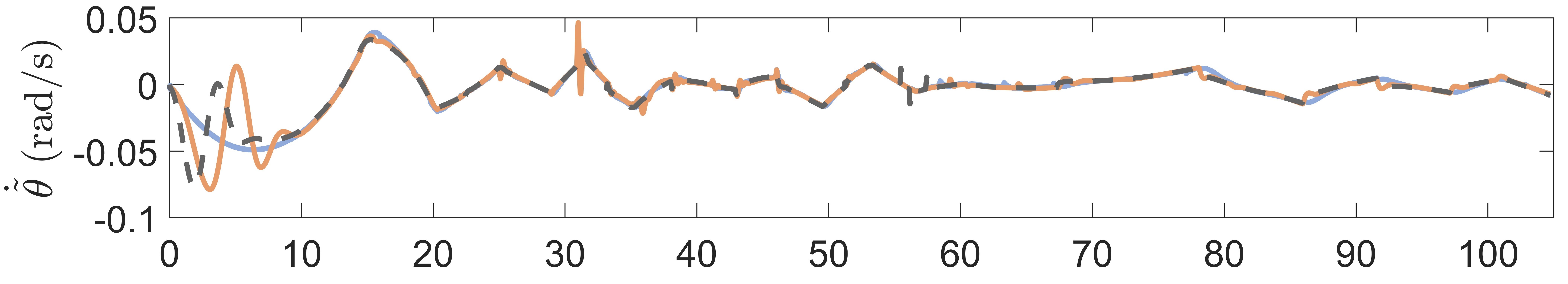}}
{\includegraphics[width=0.49\textwidth]{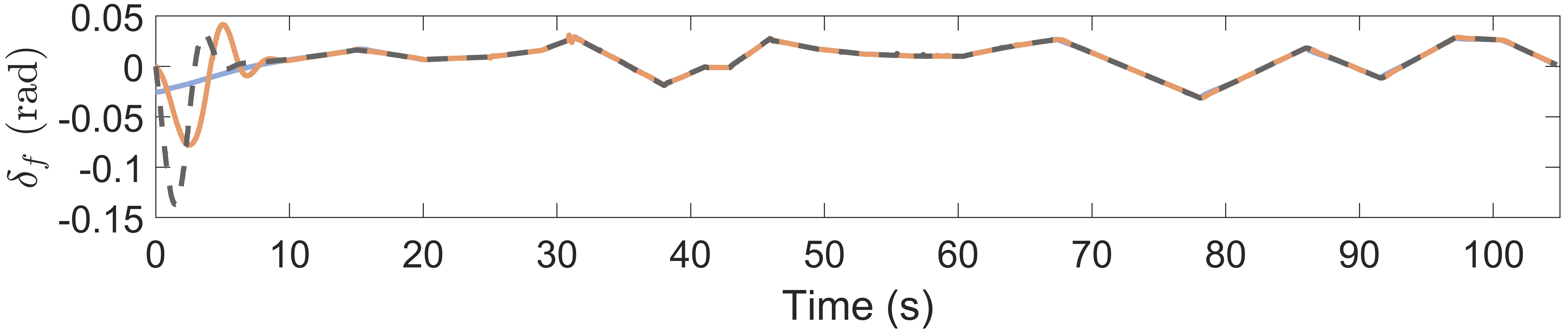}}
    \caption{Results of the nominal test. Top to bottom: (a) lateral error, (b) lateral error rate, (c) heading error, (d) heading error rate, (e) steering angle.}
    \label{fig: nominal_test}
    \end{figure}

From Fig. \ref{fig: nominal_test}(a), the LR and INDI controllers eliminate the initial lateral error much faster than the FF controller. This is achieved by feedback linearizing the input–output relation between the steering angle and lateral error, allowing LR and INDI to act directly on the lateral error. Moreover, LR explicitly uses longitudinal speed and acceleration, while INDI uses them implicitly via the $\ddot e_{lat}$ measurement. Consequently, both controllers deliver accurate and consistent performance across a broad range of speeds and accelerations. In contrast, the FF controller exhibits larger lateral errors at low speeds than at high speeds, highlighting the value of integrating longitudinal motion awareness into the lateral controller. Although heading error and its rate are not explicitly regulated external states for LR and INDI, Fig. \ref{fig: nominal_test}(c)–(d) shows that they remain bounded, consistent with Proposition \ref{prop: internal_ISS}.

\subsection{Robustness test}
In this subsection, we assess the robustness of the LR and INDI controllers to parametric uncertainty and compare them with the baseline FF controller. To this end, four vertex cases from the uncertainty set, listed in Table \ref{table: robustness_cases}, are analyzed. These cases correspond to extreme combinations of the admissible parameter ranges, yielding a conservative robustness assessment. The speed profile in Fig. \ref{fig: vel_prof} is used.
\begin{table}[ht]
\centering
\scriptsize
\caption{Four Vertex Cases for Robustness Evaluation}
\begin{tabular}{c c c c c}
\hline\hline
\textbf{Case} & $m$ & $I_z$ & $C_f$ & $C_r$ \\
\hline
1 & $1.3\hat{m}$ & $1.3\hat{I}_z$ & $1.8\hat{C}_f$ & $0.2\hat{C}_r$ \\
2 & $\hat{m}$    & $\hat{I}_z$    & $1.8\hat{C}_f$ & $0.2\hat{C}_r$ \\
3 & $1.3\hat{m}$ & $1.3\hat{I}_z$ & $0.2\hat{C}_f$ & $1.8\hat{C}_r$ \\
4 & $\hat{m}$    & $\hat{I}_z$    & $0.2\hat{C}_f$ & $1.8\hat{C}_r$ \\
\hline\hline
\end{tabular}
\label{table: robustness_cases}
\end{table}

For illustration, the evolution of the external states and steering angle of Case 1 is depicted below in Fig. \ref{fig: robust_test}.
    \begin{figure}[h]
    \centering
    {\includegraphics[width=0.49\textwidth]{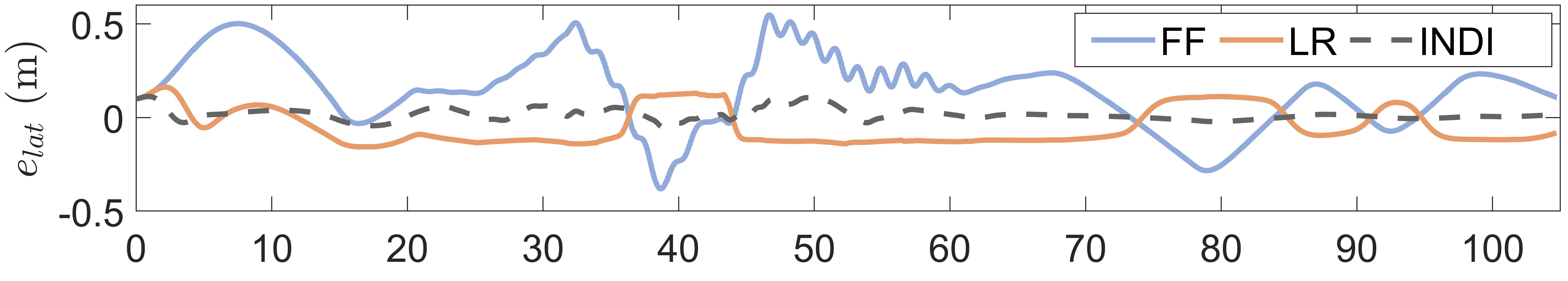}}
    {\includegraphics[width=0.49\textwidth]{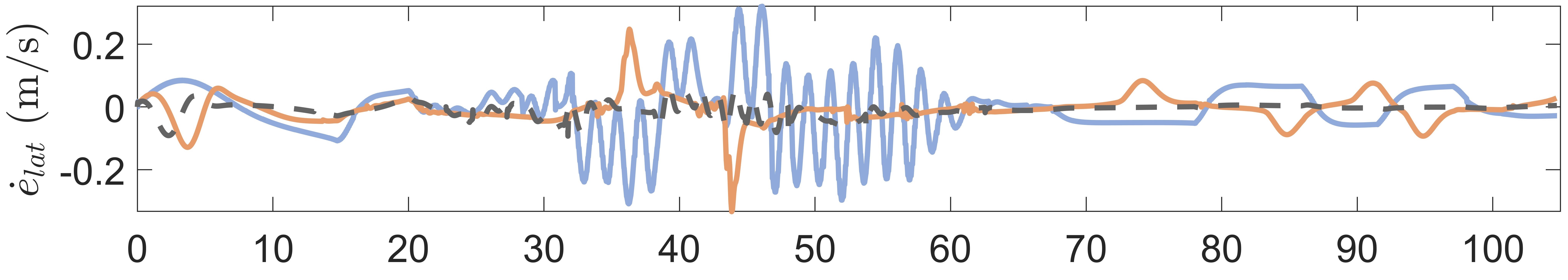}}
{\includegraphics[width=0.49\textwidth]{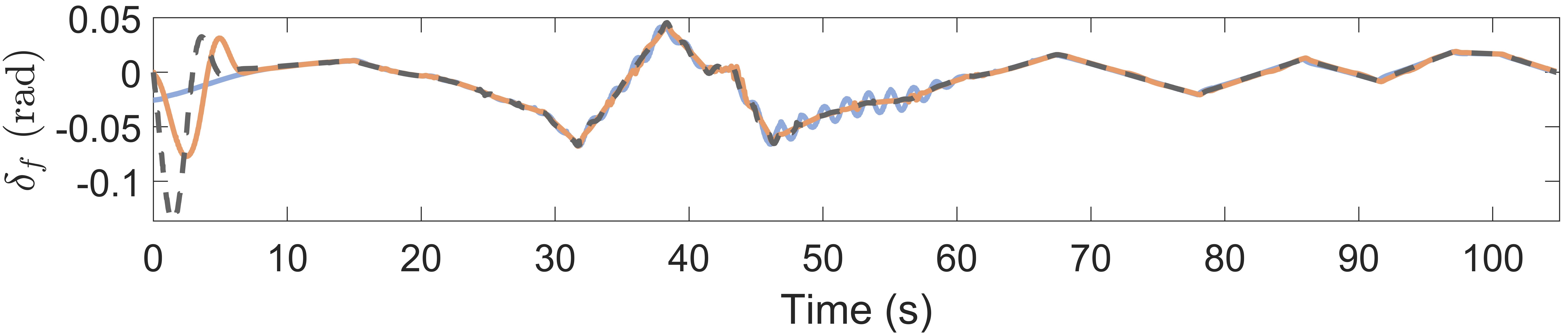}}
    \caption{Comparison of controllers in robustness test Case 1. Top to bottom: (a) lateral error, (b) lateral error rate, (c) steering angle.}
    \label{fig: robust_test}
    \end{figure}

  As shown in Fig. \ref{fig: robust_test}(a)–(b), in Case 1 the LR and INDI controllers achieve substantially better tracking accuracy than the FF controller and display less oscillatory behavior, indicating enhanced robustness to parametric uncertainty.

For a quantitative comparison, Table \ref{table: rms_cases} reports the root-mean-square (RMS) values of \(e_{lat}\), \(\dot e_{lat}\), and \(\|\boldsymbol{\xi}\|_2\) for the controllers in the vertex cases. RMS serves as a practical indicator of tracking accuracy and robustness, as it is difficult to infer from data whether the states have reached their ultimate bounds. The table also shows the relative RMS reductions of the proposed controllers compared to the baseline FF controller. Moreover, Figure \ref{fig: radar_rms} provides a compact visual comparison of the \(\|\boldsymbol{\xi}\|_2\) RMS values across the vertex cases.
\begin{table}[ht]
\centering
\scriptsize
\caption{RMS Comparison of Controllers Across Vertex Cases}
\begin{tabular}{ccc cc}
\hline
\hline
\textbf{Case} & \textbf{Controller} & $\boldsymbol{e}_{lat}~\textbf{RMS} $ & $ \dot{\boldsymbol{e}}_{lat}~\textbf{RMS}$ & $ \|\boldsymbol{\xi}\|_2~\textbf{RMS}$
 \\
\hline
\multirow{3}{*}{1} & FF & 0.2407 $(\downarrow 0\%)$ & 0.0885 $(\downarrow 0\%)$ & 0.2565 $(\downarrow 0\%)$
\\
                   & LR     & 0.1082 $(\downarrow 55\%)$ & 0.0458 $(\downarrow 48\%)$ & 0.1175 $(\downarrow 54\%)$ \\
                   & INDI   & \textbf{0.0341} $(\downarrow 86\%)$ & \textbf{0.0226} $(\downarrow 74\%)$ & \textbf{0.0409} $(\downarrow 84\%)$\\
\hline
\multirow{3}{*}{2} & FF & 0.2409 $(\downarrow 0\%)$ & 0.0871 $(\downarrow 0\%)$ & 0.2562 $(\downarrow 0\%)$
\\
                   & LR     & 0.1082 $(\downarrow 55\%)$ & 0.0458 $(\downarrow 47\%)$ & 0.1175 $(\downarrow 54\%)$ \\
                   & INDI   & \textbf{0.0341} $(\downarrow 86\%)$ & \textbf{0.0226} $(\downarrow 74\%)$ & \textbf{0.0409} $(\downarrow 84\%)$\\
\hline
\multirow{3}{*}{3} & FF & 0.4412 $(\downarrow 0\%)$ & 0.1494 $(\downarrow 0\%)$ &  0.4658 $(\downarrow 0\%)$
\\
                   & LR     & 0.2971 $(\downarrow 33\%)$ & 0.0903 $(\downarrow 40\%)$ & 0.3105 $(\downarrow 33\%)$\\
                   & INDI   & \textbf{0.0311} $(\downarrow 93\%)$ & \textbf{0.0180} $(\downarrow 88\%)$ & \textbf{0.0360} $(\downarrow 92\%)$\\
\hline
\multirow{3}{*}{4} & FF & 0.4397 $(\downarrow 0\%)$ & 0.1492 $(\downarrow 0\%)$ & 0.4643 $(\downarrow 0\%)$
\\
                   & LR     & 0.2961 $(\downarrow 33\%)$ & 0.0901 $(\downarrow 40\%)$ & 0.3095 $(\downarrow 33\%)$\\
                   & INDI   & \textbf{0.0310} $(\downarrow 93\%)$ & \textbf{0.0179} $(\downarrow 88\%)$ & \textbf{0.0360}$(\downarrow 92\%)$\\
\hline
\hline
\end{tabular}
\label{table: rms_cases}
\end{table}
\begin{figure}[htb!]
    \centering
    \includegraphics[width=0.2\textwidth]{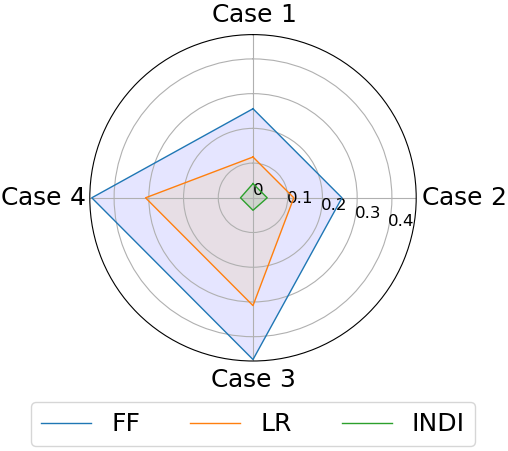}
    \caption{$\|\boldsymbol{\xi}\|_2$ RMS radar chart: controller comparison in vertex cases}
    \label{fig: radar_rms}
\end{figure}

Table \ref{table: rms_cases} and Fig. \ref{fig: radar_rms} illustrate that both LR and INDI significantly outperform the baseline in every vertex case. Among the three controllers, INDI consistently achieves the lowest RMS values, with reductions of up to 93\% in \(e_{lat}\), 88\% in \(\dot e_{lat}\), and 92\% in \(\|\boldsymbol{\xi}\|_2\) relative to the baseline. LR also yields substantial improvements, with reductions of up to 55\% in \(e_{lat}\), 48\% in \(\dot e_{lat}\), and 54\% in \(\|\boldsymbol{\xi}\|_2\). Mass and moment of inertia variations (Cases 1 and 3 vs. Cases 2 and 4) appear to have a limited impact on the performance of all controllers. Notably, LR is relatively more sensitive to changes in cornering stiffness (Cases 1-2 vs. 3-4), while INDI maintains robust and consistent performance across all cases.

\subsection{Effect of robustness-related parameters}
As discussed in Remark~\ref{rem: LR_epsilon}, reducing $\epsilon$ enhances the robustness of LR. While the effect is less conclusive for INDI, Remark~\ref{rem: INDI_tau} indicates that decreasing $\tau$ may offer a similar robustness benefit. In this subsection, we study how these two parameters influence the trade-off between robustness and control smoothness, using the vertex cases of vehicle parameters in Table \ref{table: robustness_cases} and the speed profile in Fig. \ref{fig: vel_prof}.

For the LR controller, we examine $\epsilon$ values of 0.25, 0.35 (default), and 0.45. The INDI controller is tested with $\tau$ values of 0.01s, 0.03s (default), and 0.05s. Figs. \ref{fig: LR_test} and \ref{fig: INDI_test} depict the evolution of external states and steering for the LR controllers and INDI controllers, respectively, under Case 1.
 \begin{figure}[h]
    \centering
    {\includegraphics[width=0.49\textwidth]{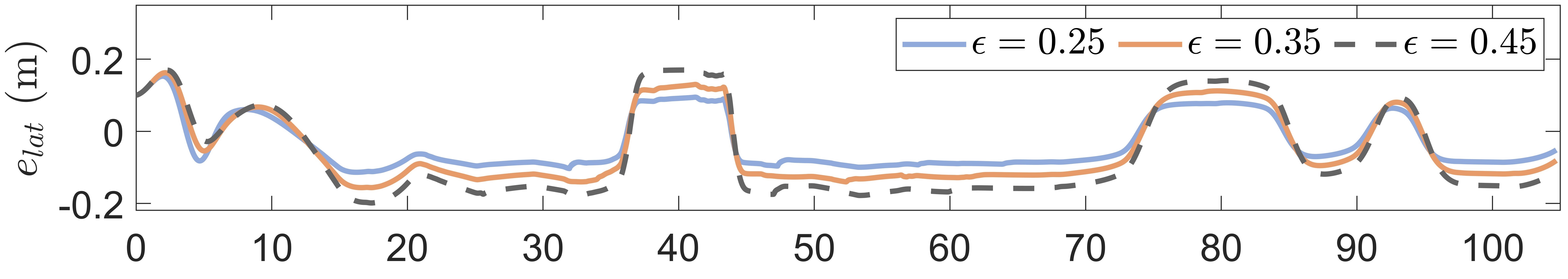}}
    {\includegraphics[width=0.49\textwidth]{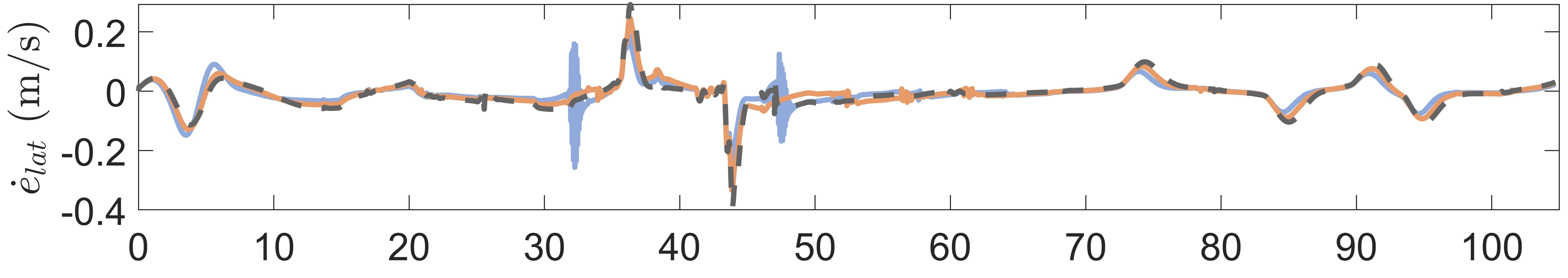}}
{\includegraphics[width=0.49\textwidth]{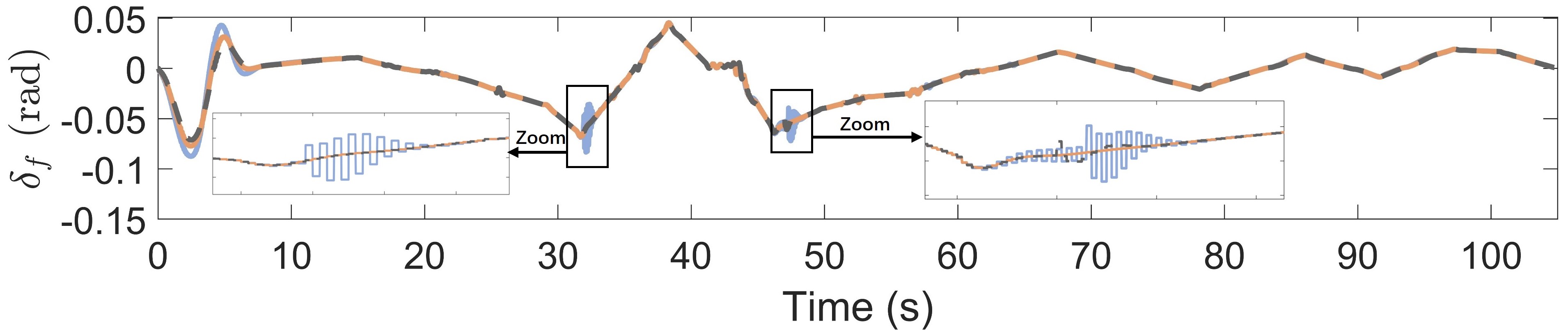}}
    \caption{Comparison of different $\epsilon$ for LR in Case 1. Top to bottom: (a) lateral error, (b) lateral error rate, (c) steering angle.}
    \label{fig: LR_test}
    \end{figure}
\begin{figure}[h]
    \centering
    {\includegraphics[width=0.49\textwidth]{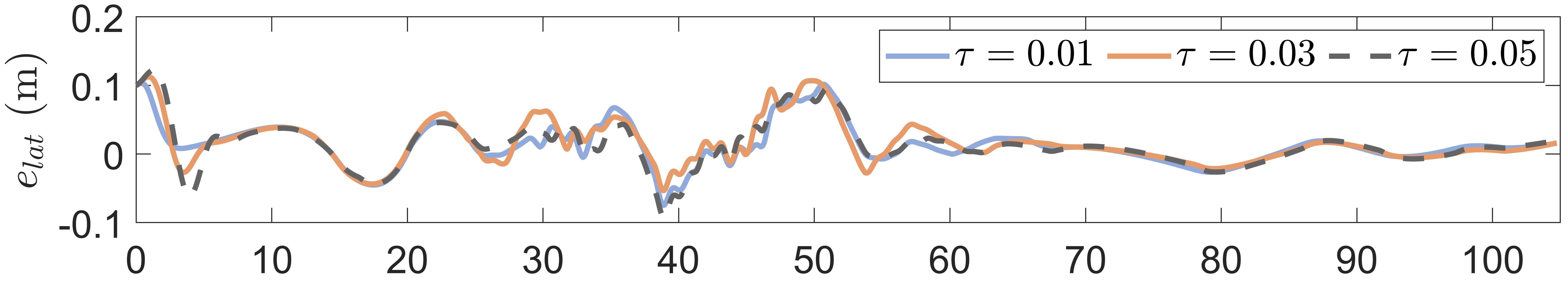}}
    {\includegraphics[width=0.49\textwidth]{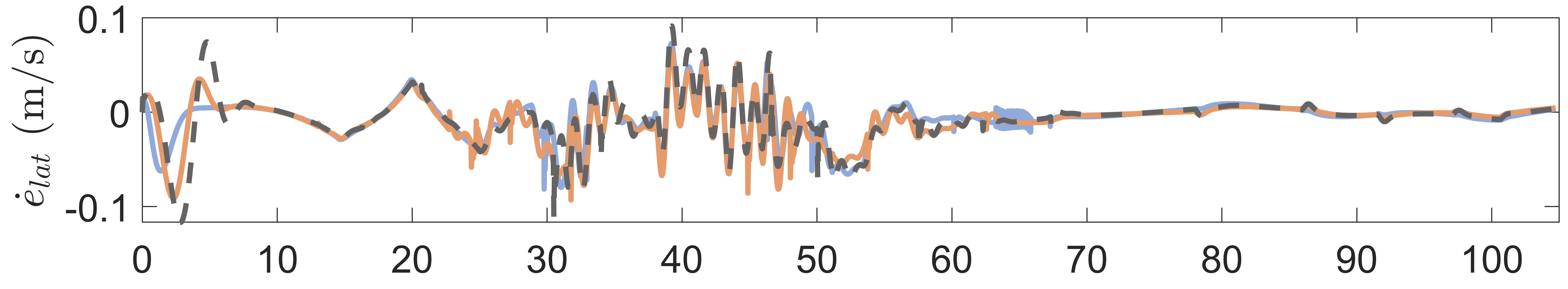}}
{\includegraphics[width=0.49\textwidth]{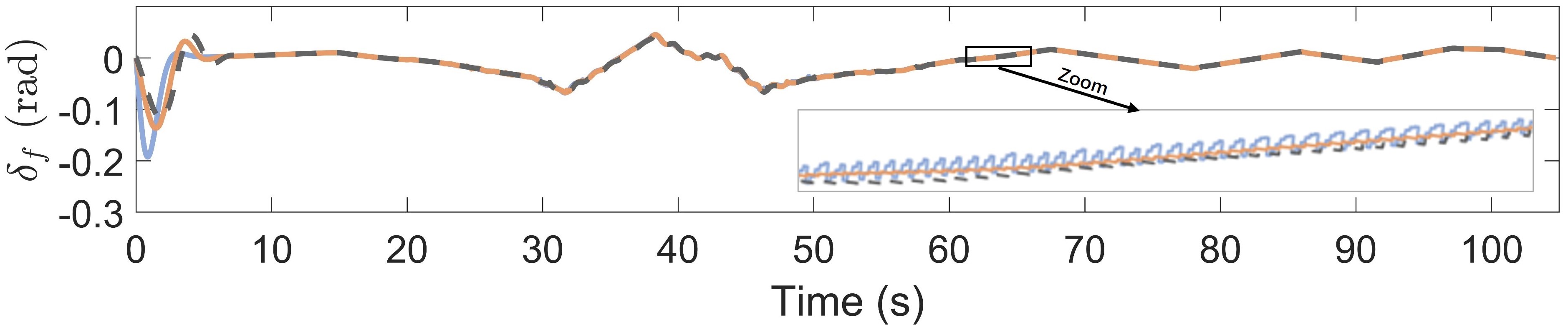}}
    \caption{Comparison of different $\tau$ for INDI in Case 1. Top to bottom: (a) lateral error, (b) lateral error rate, (c) steering angle.}
    \label{fig: INDI_test}
    \end{figure}

In Fig. \ref{fig: LR_test}(a), a noticeable decrease in lateral error $e_{lat}$ is observed as the value of $\epsilon$ decreases. However, when $\epsilon$ reaches 0.25, significant chattering effects appear, as illustrated in the zoomed-in view of Fig. \ref{fig: LR_test}(c). While Fig. \ref{fig: INDI_test} does not exhibit an equally clear trend in accuracy with decreasing $\tau$, Fig. \ref{fig: INDI_test}(c) reveals that reducing $\tau$ to 0.01s also results in slight chattering.

For a clearer comparison, we compute the RMS of the external state 2-norm, $\|\boldsymbol{\xi}\|_2$, over all four vertex cases. The results, shown in Fig.~\ref{fig: param_sweep_radar}, use separate radar charts to compare different $\epsilon$ values for LR and $\tau$ values for INDI. As seen in the figure, the RMS of $\|\boldsymbol{\xi}\|_2$ decreases with smaller $\epsilon$ for LR and smaller $\tau$ for INDI, consistent with Remarks \ref{rem: LR_epsilon} and \ref{rem: INDI_tau}.
\begin{figure}[h]
    \centering
    \subfigure[LR controllers with different $\epsilon$]{\includegraphics[width=0.24\textwidth]{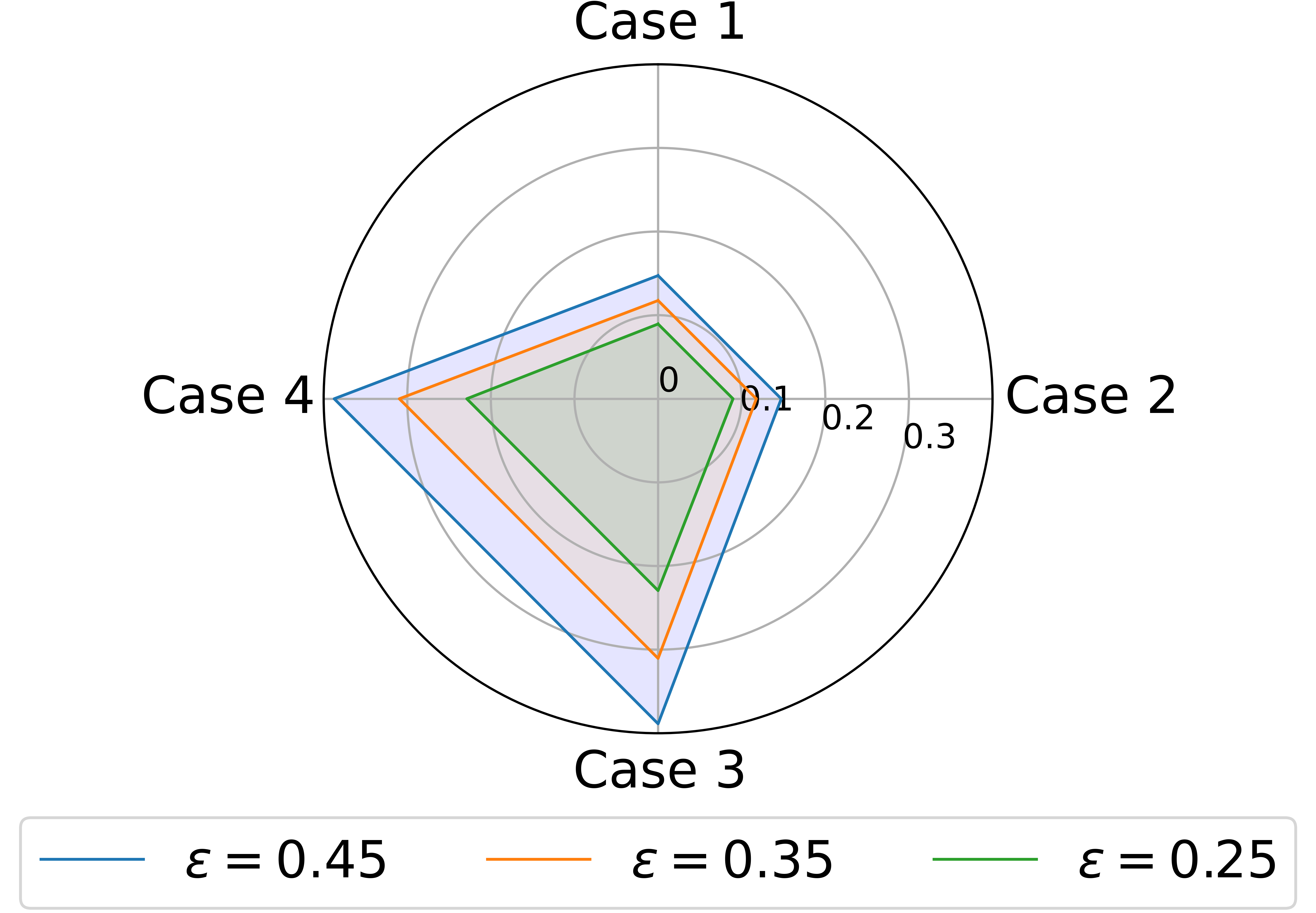}}
    \subfigure[INDI controllers with different $\tau$]{\includegraphics[width=0.24\textwidth]{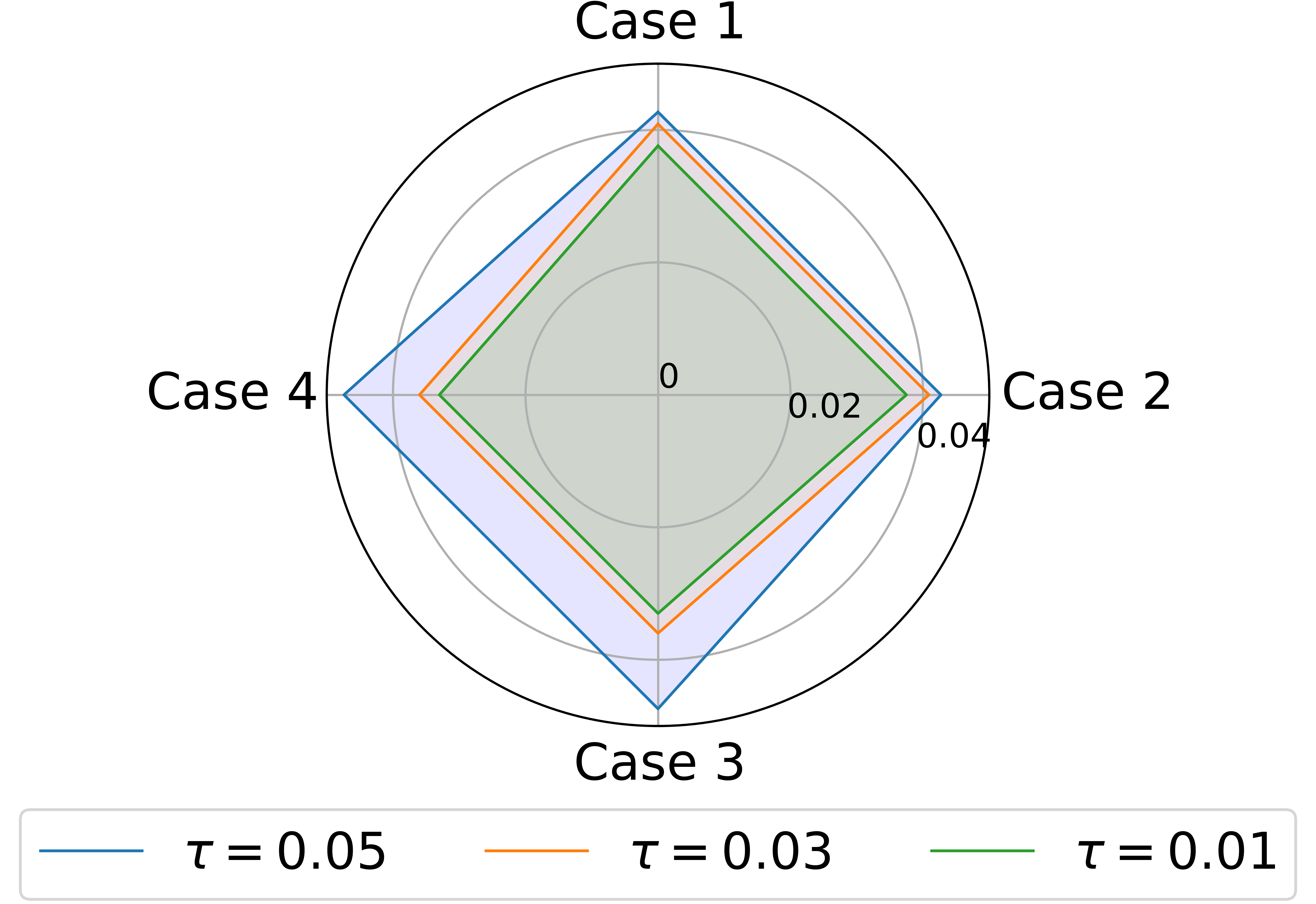}}
    \caption{$\|\boldsymbol{\xi}\|_2$ RMS radar chart: effect of different $\epsilon$ and $\tau$ in vertex cases}
    \label{fig: param_sweep_radar}
    \end{figure}

    Overall, both parameters exhibit a performance--smoothness trade-off: smaller values improve robustness and tracking accuracy, but may also lead to less smooth steering responses.

    \section{Real-Vehicle Implementation and Validation}
    \label{sec veh}
While the broader robustness study is performed in simulation, where speeds, parameter uncertainties, and controller parameters can be safely and extensively varied, an additional real-vehicle test is added here to validate the controllers on real hardware under realistic sensing, actuation, and timing constraints. Because systematically altering tire cornering stiffness—the main source of lateral-dynamics uncertainty \cite{peng1990lateral}—is difficult in practice, and moderate payload variations have only a minor effect, this experiment focuses on practical implementability and path-tracking performance.

The controllers are experimentally validated on a Chevrolet Bolt EUV. As shown in Fig.~\ref{fig: real_vehicle_platform}(a), the vehicle is equipped with LiDAR, cameras, and GPS antennas; an onboard IMU is also used. Computation is performed on an Intel M50FCP server with 64 GB RAM and dual Intel Xeon Gold 5420+ processors, running a ROS2 Humble software stack. The vehicle’s nominal parameters \cite{opendbc_bolt_params} and the corresponding parameter bounds used by the controllers are summarized in Table~\ref{table: real_vehicle_parameters}.
\begin{figure}[htb!]
    \centering
    \subfigure[Chevrolet Bolt EUV]{
    \begin{minipage}[t]{0.24\textwidth}
        \centering
        \includegraphics[width=1\textwidth]{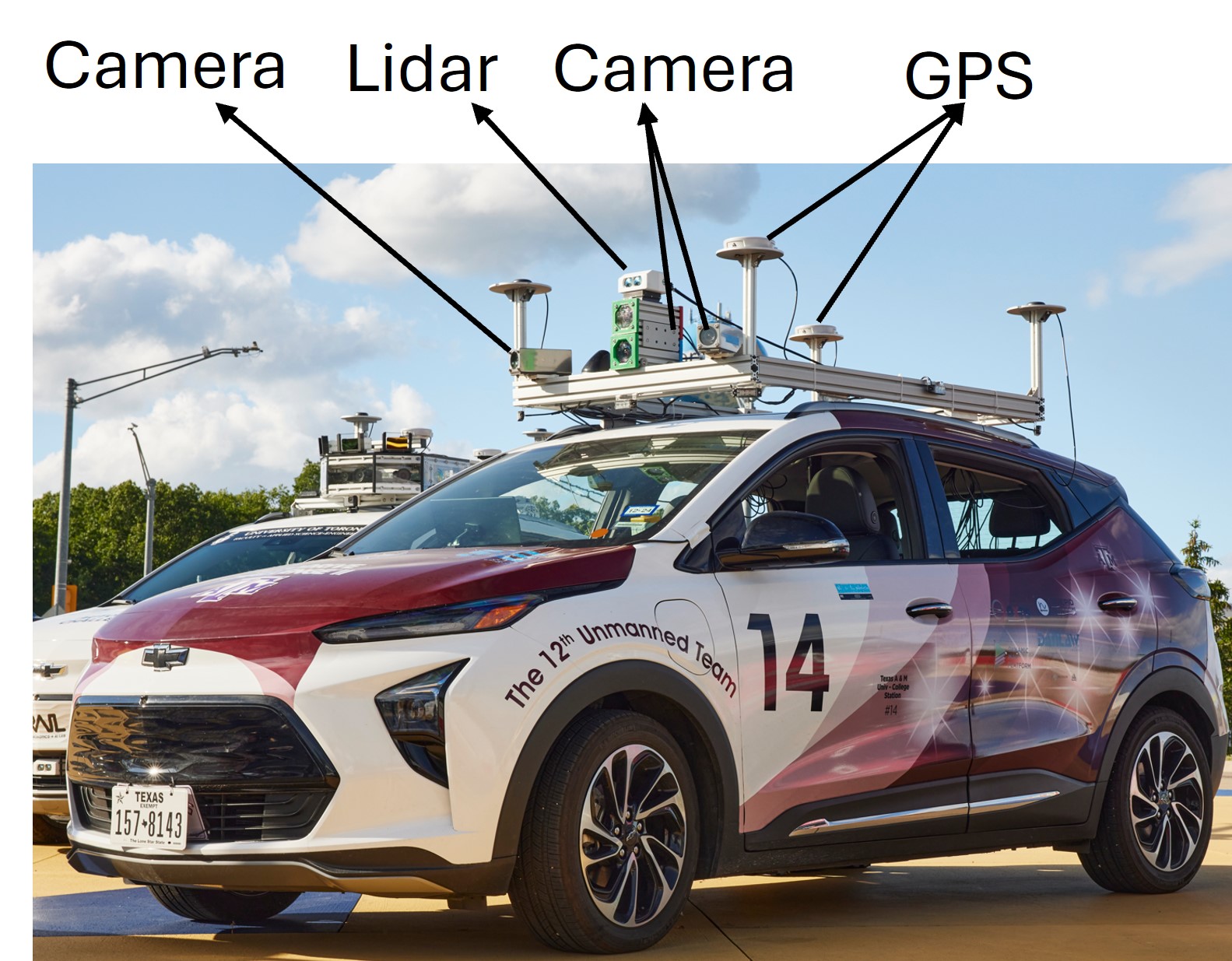}
    \end{minipage}
}\hfill
    \subfigure[Schematic of RELLIS test site]{
    \begin{minipage}[t]{0.2\textwidth}
        \centering
        \includegraphics[width=0.43\textwidth]{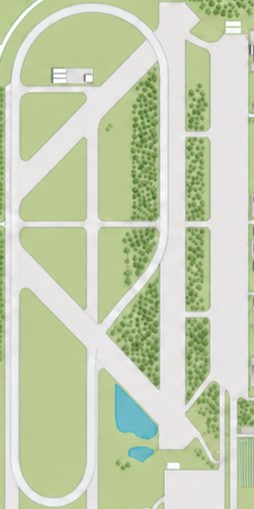}
    \end{minipage}
}
    \caption{Vehicle platform and test site}
    \label{fig: real_vehicle_platform}
\end{figure}
\begin{table}[ht]
\centering
\caption{Parameters of the Chevrolet Bolt EUV}
\resizebox{1\linewidth}{!}{
\begin{tabular}{cccccc}
\hline
\hline
\textbf{Parameter} & \textbf{Symbol} & \textbf{Unit} & \textbf{Nominal} & \makecell{\textbf{Min}\\\textbf{(relative)}} & \makecell{\textbf{Max}\\\textbf{(relative)}} \\
\hline
Mass & $m$ & kg & 1805 & 1 & 1.3 \\
Moment of Inertia & $I_z$ & kg$\cdot$m$^2$ & 2720 & 1 & 1.3 \\
Front Cornering Stiffness & $C_{f}$ & N/rad & 237230 & 0.2 & 1.8 \\
Rear Cornering Stiffness & $C_{r}$ & N/rad & 250009 & 0.2 & 1.8 \\
C.G. to front axle & $a$ & m & 1.055  & -- & -- \\
C.G. to rear axle & $b$ &m & 1.583  & -- & -- \\
\hline
\hline
\label{table: real_vehicle_parameters}
\end{tabular}}
\end{table}

The experiment is conducted at TAMU's RELLIS campus, as shown in Fig. \ref{fig: real_vehicle_platform}(b). The controller parameters are adjusted on the vehicle until further changes no longer improve tracking accuracy without causing noticeable oscillations; the final values are listed in Table \ref{table: veh_control_praram}. For safety, the speed profile in Fig. \ref{fig: vel_prof} is scaled down by a factor of 7, yielding a maximum speed of \(4.28\,\mathrm{m/s}\). Due to space constraints, the reference path in Fig. \ref{fig: track} is scaled by the same factor while maintaining its heading profile and shape, leading to a peak curvature around \(0.085\,\mathrm{m}^{-1}\), comparable to urban driving conditions.
\begin{table}[ht]
\centering
\scriptsize
\caption{Controller Parameters for the Real-Vehicle Experiment}
\begin{tabular}{c c c c| c c c| c c c}
\hline\hline
\multicolumn{4}{c|}{\textbf{LR Controller}} & \multicolumn{3}{c|}{\textbf{INDI Controller}}  & \multicolumn{3}{c}{\textbf{FF Controller}} \\
\hline
$k_1$ & $k_2$ & $\epsilon$ & $\rho_0$ & $k_1$ & $k_2$ & $\tau$ & $k_e$ & $k_\theta$ & $k_{\dot\theta}$\\
11.52  & 1.54  & 0.25       & 0.50      & 0.12  & 4.20 & 0.01 & 0.30&0.15 &0.08 \\
\hline\hline
\end{tabular}
\label{table: veh_control_praram}
\end{table}

Figure \ref{fig: real_vehicle_results} shows the real-vehicle experimental results. Overall, the proposed LR and INDI controllers achieve better tracking accuracy than the FF controller, with INDI yielding the smallest overall error in Fig. \ref{fig: real_vehicle_results}(a). Their advantage is especially evident in high-curvature sections at \(t\approx 32\text{--}48\,\mathrm{s}\), around \(t\approx 70\,\mathrm{s}\), and at \(t\approx 80\text{--}103\,\mathrm{s}\). Notably, at lower speeds, FF exhibits worse tracking performance than at higher speeds, with a prolonged large-error segment over \(t\approx 80\text{--}90\,\mathrm{s}\) and a peak exceeding \(0.6\,\mathrm{m}\) near \(t\approx 95\,\mathrm{s}\). Meanwhile, LR and INDI maintain more consistent performance over the entire test, highlighting the benefit of incorporating longitudinal motion awareness into the design. These experiments confirm that the proposed controllers can be implemented on real hardware and deliver superior path-tracking performance in practice.
 \begin{figure}[h]
    \centering
    {\includegraphics[width=0.49\textwidth]{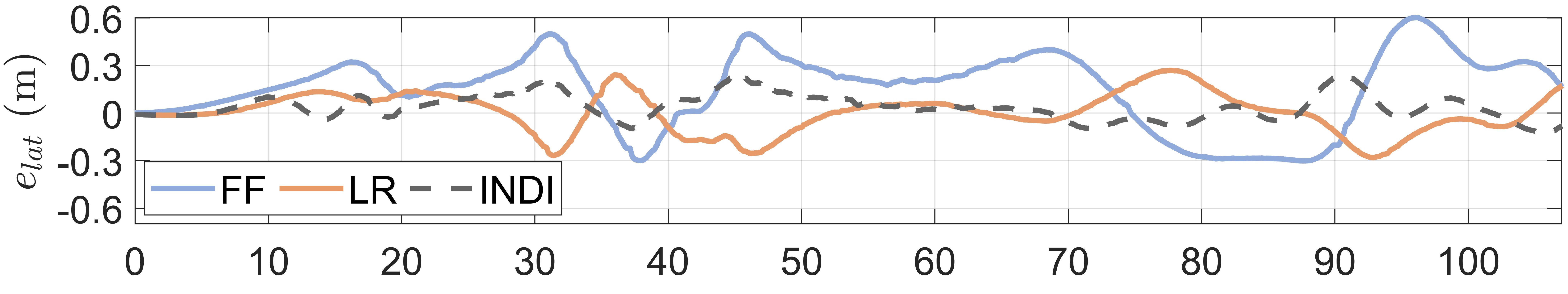}}
    {\includegraphics[width=0.49\textwidth]{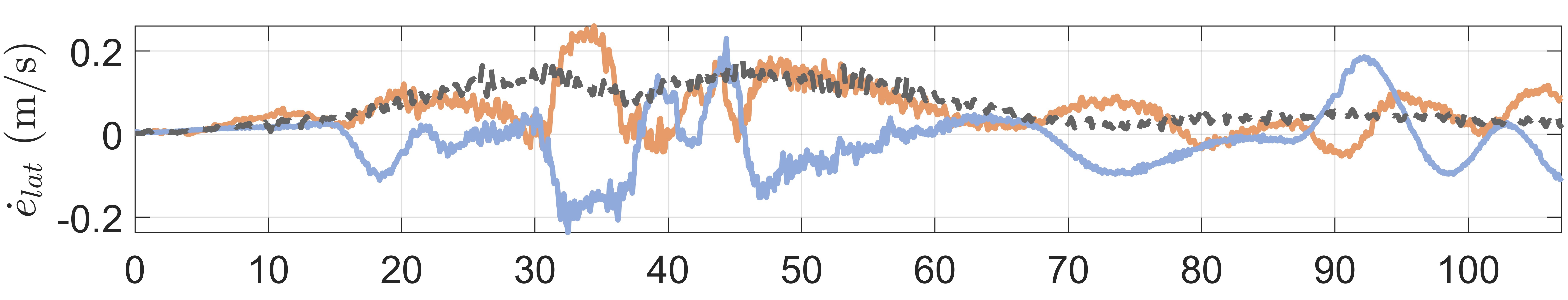}}
{\includegraphics[width=0.49\textwidth]{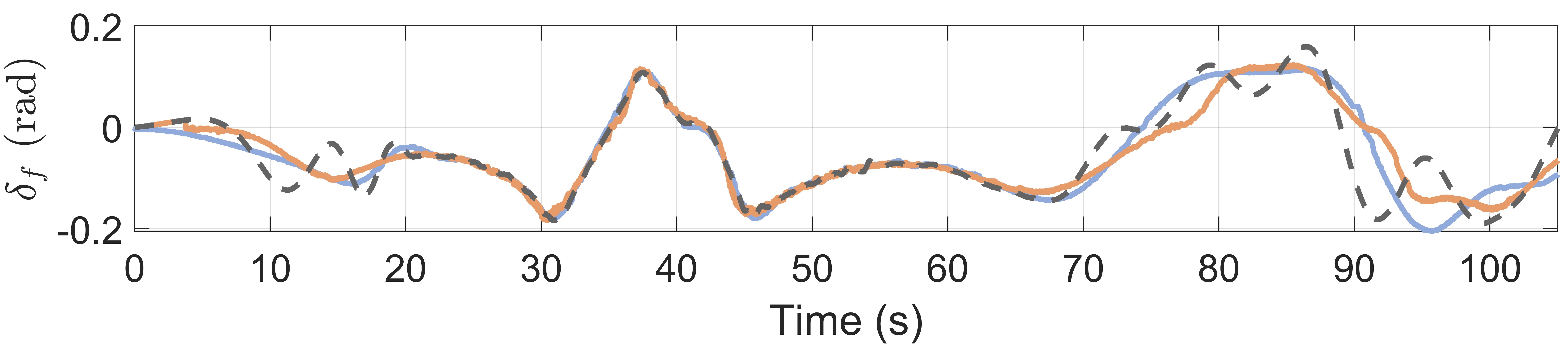}}
    \caption{Real-vehicle experiment results. Top to bottom: (a) lateral error, (b) lateral error rate, (c) steering angle.}
    \label{fig: real_vehicle_results}
    \end{figure}

     \section{Conclusion}
    \label{sec6}
This study proposed a robust nonlinear lateral control framework for AVs that explicitly accounts for longitudinal motion. A tracking error model was derived to capture the effects of varying speed and acceleration. Using this model, feedback linearization produced a control law that directly uses real-time longitudinal motion information. The internal dynamics were examined to ensure overall system stability.

To address parametric uncertainties, two alternative robustification strategies were proposed: a sliding-mode--inspired LR controller and an INDI controller. Both were analytically proven to guarantee robustness in terms of ultimate boundedness. The LR controller employs a novel structure based on an equivalent control term, while the INDI analysis differs from prior work by guaranteeing robustness without Jacobian linearization or assuming globally bounded perturbations. Together, these methods give practitioners flexibility in trading off parameter knowledge against measurement availability.

Simulation results verified the effectiveness of the proposed methods. Incorporating longitudinal motion awareness enhanced tracking accuracy and consistency under varying speeds and accelerations, and both controllers exhibited improved robustness, with INDI performing best overall. As anticipated, robustness increased for the LR and INDI controllers as $\epsilon$ and $\tau$ were reduced, respectively. Real-vehicle experiments further showed that the controllers can be practically implemented and improve path-tracking performance. Overall, the results highlight the benefits of integrating longitudinal motion awareness with robust design for AV lateral control.

\section*{Acknowledgment}
The authors thank General Motors and SAE International for sponsoring the SAE AutoDrive Challenge II, which supplied the Chevrolet Bolt EUV platform and inspired this work.

\bibliographystyle{IEEEtran}
\bibliography{IEEEabrv,root}

\end{document}